\documentclass[journal]{IEEEtran}

\usepackage{latexsym,amssymb,amsmath,graphicx,epsf,cite,bbm}
\usepackage{epsfig}
\usepackage{graphicx}
\usepackage{algorithm,algorithmic}
\usepackage{amsmath}
\usepackage{amsthm}
\usepackage{slashbox}
\usepackage{pict2e}
\usepackage{subcaption}
\usepackage[export]{adjustbox}

\newcommand{\be}{\begin{equation}}
\newcommand{\ee}{\end{equation}}
\newcommand{\bea}{\begin{eqnarray}}
\newcommand{\eea}{\end{eqnarray}}

\renewcommand{\vec}[1]{ \mbox{$\mathbf {#1}$}}
\DeclareMathOperator*{\argmin}{\arg\!\min}

\newcommand{\ei}{\end{itemize}}
\newcommand{\bi}{\begin{itemize}}

\newcommand{\MB}{\left[\begin{array}}
	\newcommand{\ME}{\end{array}\right]}

\newtheorem{theorem}{Theorem}

\newtheorem{remark}{Remark}

\begin{document}
	\title{Sequential Outlier Detection based on Incremental Decision Trees}
	\author{ Mohammadreza Mohaghegh Neyshabouri and Suleyman S. Kozat, {\em Senior Member, IEEE} 
	\thanks{
		This work is supported in part by Turkish Academy of Sciences Outstanding Researcher Programme, TUBITAK Contract No. 117E153.
		
		M. Mohaghegh N. and S. S. Kozat are with the Department of Electrical and Electronics Engineering, Bilkent University, Ankara, Turkey, e-mail: \{mohammadreza, kozat\}@ee.bilkent.edu.tr, tel: +90 (312) 290-2336. 
		}}
	\maketitle
	\begin{abstract}
		We introduce an online outlier detection algorithm to detect outliers in a sequentially observed data stream. For this purpose, we use a two-stage filtering and hedging approach. In the first stage, we construct a multi-modal probability density function to model the normal samples. In the second stage, given a new observation, we label it as an anomaly if the value of aforementioned density function is below a specified threshold at the newly observed point. In order to construct our multi-modal density function, we use an incremental decision tree to construct a set of subspaces of the observation space. We train a single component density function of the exponential family using the observations, which fall inside each subspace represented on the tree. These single component density functions are then adaptively combined to produce our multi-modal density function, which is shown to achieve the performance of the best convex combination of the density functions defined on the subspaces. As we observe more samples, our tree grows and produces more subspaces. As a result, our modeling power increases in time, while mitigating overfitting issues. In order to choose our threshold level to label the observations, we use an adaptive thresholding scheme. We show that our adaptive threshold level achieves the performance of the optimal pre-fixed threshold level, which knows the observation labels in hindsight. Our algorithm provides significant performance improvements over the state of the art in our wide set of experiments involving both synthetic as well as real data.
	\end{abstract}
	\begin{center}
		\bfseries EDICS Category:  MLR-SLER, MLR-APPL, MLR-LEAR.
	\end{center}
	
	\section{Introduction}\label{intro}
	\subsection{Preliminaries}
	We study sequential outlier or anomaly detection \cite{Chandola}, which has been extensively studied due to its applications in a wide set of problems from network anomaly detection \cite{netano1,netano2,netano3} and fraud detection \cite{fraud3} to medical anomaly detection \cite{medical1} and industrial damage detection \cite{industrial}. In the sequential outlier detection problem, at each round $t$, we observe a sample vector $\vec{x}_t\in \mathbb{X}$ and label it as ``normal'' or ``anomalous" based on the previously observed sample vectors, i.e., $\vec{x}_{t-1},...,\vec{x}_1$, and their possibly revealed true labels. After we declare our decision, we may or may not observe the true label of $\vec{x}_t$. The objective is to minimize the number of mislabeled samples.\par 
	
	For this purpose, we use a two-stage ``filtering" and ``hedging" method \cite{raginsky}. In the ``filtering" stage, we build in an online manner ``a model" for ``normal" samples based on the information gained from the previous rounds. Then, in the ``hedging" stage, we decide on the label of the new sample based on its conformity to our model of normal samples. A common approach in constructing the aforementioned model is to assume that the normal data is generated from an independent and identically distributed (i.i.d.) random sequence \cite{Chandola}. Hence, in the first stage of our algorithm, we model the normal samples using a probability density function, which can also be considered as a scoring function \cite{raginsky}. However, note that the true underlying model of the normal samples can be arbitrary in general (or may not even exist) \cite{Chandola}. Therefore, we approach the problem in a competitive algorithm framework \cite{competitive}. In this framework, we define a class of models called ``competition class" and aim to achieve the performance of the best model in this class. Selecting a rich class of powerful models as the competition class enables us to perform well in a wide set of scenarios \cite{competitive}. Hence, as detailed later, we choose a strong set of probability functions to compete against and seek to sequentially learn the best density function which fits to the normal data. Hence, while refraining from making any statistical assumptions on the underlying model of the samples, we guarantee that our performance is (at least) as well as the best density function in our competition class. \par 
	
	We emphasize that there exist nonparametric algorithms for density estimation \cite{ozkan1}, the parametric approaches have recently gained more interest due to their faster convergence \cite{ozkan2,kaan}. However, the parametric approaches fail if the assumed model is not capable of modeling the true underlying distribution \cite{competitive}. In this context, exponential-family distributions \cite{exfam} have attracted significant attention, since they cover a wide set of parametric distributions \cite{raginsky}, and successfully approximate a wide range of nonparametric probabilistic models as well\cite{barron}. However, single component density functions are usually inadequate to model the data in highly challenging real life applications \cite{igmm}. In this paper, in order to effectively model multi-modal distributions, we partition the space of samples into several subspaces using highly effective and efficient hierarchical structures, i.e., decision trees \cite{decisiontree}. The observed samples, which fall inside each subspace are fed to a single component exponential-family density estimator. We adaptively combine all these estimators in a mixture-of-experts framework \cite{mixofex} to achieve the performance of their best convex combination.\par
	
	We emphasize that the main challenge using a partitioning approach for multi-modal density estimation is to define a proper partition of the space of samples \cite{igmm}. Here, instead of sticking to a pre-fixed partition, we use an incremental decision tree \cite{decisiontree} approach to partition the space of samples in a nested structure. Using this method we avoid overtraining, while efficiently modeling complex distributions composed of a large number of components \cite{decisiontree}. As the first time in the literature, in order to increase our modeling power with time, we apply a highly powerful incremental decision tree \cite{decisiontree}. Using this incremental tree, whenever we believe that the samples inside a subspace cannot belong to a single component distribution, we split the subspace into two disjoint subspaces and start training two new single component density estimators on the recently emerged subspaces. Hence, our modeling power can potentially increase with no limit (and increase if needed), while mitigating the overfitting issues. \par

	In order to decide on the label of a given sample, as widely used in the literature \cite{raginsky}, we evaluate the value of our density function in the new data point $\vec{x}_t$ and compare it against a threshold. If probability density is lower than the threshold, the sample is labeled as anomalous. While this is a shown to be an effective strategy for anomaly detection, setting the threshold is a notoriously difficult problem \cite{raginsky}. Hence, instead of committing to a fixed threshold level, we use an adaptive thresholding scheme and update the threshold level whenever we receive a feedback on the true label of the samples. We show that our thresholding scheme achieves the performance of the best fixed threshold level selected in hindsight.
	
	\subsection{Prior Art and Comparisons}
	Various anomaly detection methods have been proposed in the literature that utilize Neural Networks \cite{nn1}, Support Vector Machines \cite{svm3}, Nearest Neighbors \cite{nnb1}, clustering \cite{clustering1} and statistical methods including parametric \cite{par1} and nonparametric \cite{npar1} density estimation. In the case when the normal data conform to a probability density function, the anomaly detection algorithms based on the parametric density estimation method are shown to provide superior performance \cite{evaluation}. For this reason, we adopt the parametric probability estimation based approach. In \cite{raginsky}, authors have introduced an online algorithm to fit a single component density function of the exponential-family distributions to the stream of data. However, since the real life distributions are best described using multi-modal PDFs rather than single component density functions \cite{okde}, we seek to fit multi-modal density functions to the observations. There are various multi-modal density estimation methods in the literature. In \cite{igmm}, authors propose a sequential algorithm to learn the multi-modal Gaussian distributions. However, as discussed in their paper, this algorithm provides satisfactory results only if a temporary coherency exists among subsequent observations. In \cite{okde}, an online variant of the well-known Kernel Density Estimation (KDE) method is proposed. However, no performance guarantees are provided for any of the algorithms. In this paper, we provide a multi-modal density estimation method using an incremental tree with strong performance bounds, which are guaranteed to hold in an individual sequence manner through a regret formulation \cite{raginsky}.\par

	Decision trees are widely studied in various applications including coding \cite{coding}, prediction \cite{prediction,herotree}, regression \cite{vanli} and classification \cite{classification}. These structures are shown to provide highly successful results due to their ability to refrain from overtraining while providing significant modeling power. In this paper, we adapt a novel notion of incremental decision trees \cite{incremental} to the density estimation framework. Using this decision tree, we train a set of single-component density estimators with carefully chosen sets of data samples. We combine these single-component estimators in an ensemble learning \cite{EL} framework to approximate the underlying multi-modal density function and show that our algorithm achieves the performance of the best convex combination of the single component density estimators defined on the, possibly infinite depth, decision tree. \par
	
	Adaptive thresholding schemes are widely used for anomaly detection algorithms based on density estimation \cite{Chandola}. While most of the algorithms in the literature do not provide guarantees for their anomaly detection performance, a surrogate regret bound of $O(\sqrt{t})$ is provided in \cite{raginsky}. However, since in real life applications the labels are revealed in a small portion of rounds \cite{semisupervised}, stronger performance guarantees are highly desirable. We provide an adaptive thresholding scheme with a surrogate regret bound of $O(\log{t})$. Hence, our algorithm steadily achieves the performance of the best threshold level chosen in hindsight.

	\subsection{Contributions}
	Our main contributions are as follows:
	\begin{itemize}
		\item For the first time in the literature, we adapt the notion of incremental decision trees to the multi-modal density estimation framework. We use this tree, which can grow to an infinite depth, to partition the observations space into disjoint subspaces and train different density functions on each subspace. We adaptively combine these density functions to achieve the performance of the best multi-modal density function defined on the tree.
		\item We provide guaranteed performance bounds for our multi-modal density estimation algorithm. Due to our competitive algorithm framework, our performance bounds are guaranteed to hold in an individual sequence manner.
		\item Due to our individual sequence perspective, our algorithm can be used in unsupervised, semi-supervised and supervised settings.
		\item Our algorithm is truly sequential, such that no a priori information on the time horizon or the number of components in the underlying probability density function is required.
		\item We propose an adaptive thresholding scheme that achieves a regret bound of $O(\log{t})$ against the best fixed threshold level chosen in hindsight. This thresholding scheme improves the state-of-the-art $O(\sqrt{t})$ regret bound provided in \cite{raginsky}.
		\item We demonstrate significant performance gains in comparison to the state-of-the-art algorithms through extensive set of experiments involving both synthetic and real data.
	\end{itemize}

	\subsection{Organization}
	In section \ref{ProbDes}, we formally define the problem setting and our notation. Next, we explain our single-component density estimation methods in section \ref{sec2}. In section \ref{BT}, we introduce our decision tree and explain how we use it to incorporate the single-component density estimators and create our multi-modal density function. Then, we explain the anomaly detection step of our algorithm in section \ref{anomalydetection}, which completes our algorithm description. In section \ref{simulations} we demonstrate the performance of our algorithm against the state-of-the-art methods on both synthetic and real data. We finish with concluding remarks in section \ref{conclusion}.
	
	\section{Problem Description}\label{ProbDes}
	In this paper, all vectors are column vectors and denoted by boldface lower case letters. For a $K$-element vector $\vec{u}$, ${u}_i$ represents the $i^{\text{th}}$ element and $\lVert \vec{u} \rVert=\sqrt{\vec{u}^T\vec{u}}$ is the $l^2$-norm, where $\vec{u}^T$ is the ordinary transpose. For two vectors of the same length $\vec{u}$ and $\vec{v}$, $\langle \vec{u},\vec{v} \rangle = \vec{u}^T\vec{v}$ represents the inner product. We show the indicator function by $\mathbf{1}_{\lbrace \text{condition} \rbrace}$, which is equal to $1$ if the condition holds and $0$ otherwise. \par 
	
	We study sequential outlier detection problem, where at each round $t \geq 1$, we observe a sample vector $\vec{x}_t \in \mathbb{R}^m$ and seek to determine whether it is anomalous or not. We label the sample vector $\vec{x}_t$ by $\hat{d}_t=-1$ for normal samples and $\hat{d}_t=1$ for anomalous ones, where $d_t$ corresponds to the true label which may or may not be revealed. In general, the cost of making an error on normal and anomalous data may not be the same. Therefore, we define ${C}_{d_t}$ as the cost of making an error while the true label is $d_t$. The objective is to minimize the accumulated cost in a series of rounds, i.e., $\sum_{t=1}^{T} {C}_{d_t}\mathbf{1}_{\lbrace \hat{d}_t \neq d_t \rbrace}$.\par
	 
	In our two step approach, we first introduce an algorithm for probability density estimation, which learns a multi-modal density function that fits ``best" to the observations. This density function can be seen as a scoring function determining the normality of samples. Due to the online setting of our problem, at each round $t$, our density function estimate, denoted by $\hat{p}_t(\cdot)$, is a function of previously observed samples and their possibly revealed labels, i.e.,
	\begin{equation}
	\hat{p}_t(\cdot)=f(\vec{x}_1,\vec{x}_2,...,\vec{x}_{t-1}, d_1,d_2,...,d_{t-1}).
	\end{equation}
	Note that in general, even if the samples are not generated from a density function, e.g., deterministic framework \cite{deterministic}, our estimate $\hat{p}_t(\cdot)$ can be seen as a scoring function determining the normality of the samples. As widely used in the literature \cite{Murphy}, we measure the accuracy of our density function estimate $\hat{p}_t$ by the log-loss function
	\begin{equation}\label{eq:logloss}
	l_P(\hat{p}_t(\vec{x}_t))=-\log(\hat{p}_t(\vec{x}_t))).
	\end{equation}
	
	In order to refrain from any statistical assumptions on the normal data, we work in a competitive framework \cite{competitive}. In this framework we seek to achieve the performance of the best model in a class of models called the competition class. We use the notion of ``regret" as our performance measure in both density estimation and anomaly detection steps. The regret of a density estimator producing the density function $\hat{p}_t(\cdot)$ against a density function $p(\cdot)$ at round $t$ is defined as
	\begin{equation}
	r_{P,t}(\hat{p}_t(\vec{x}_t),p(\vec{x}_t))=-\log(\hat{p}_t(\vec{x}_t))+\log(p(\vec{x}_t)),
	\end{equation}
	where selection of $p(\cdot)$ will be clarified later. We denote the accumulated density estimation regret up to time $T$ by
	\begin{equation}
	R_{P,T}=\sum_{t=1}^{T}r_{P,t}(\hat{p}_t(\vec{x}_t),p(\vec{x}_t)).
	\end{equation}
	\par 
	In order to produce our decision on the label of observations being ``normal" or ``anomalous", at each round $t$, we observe the new sample $\vec{x}_t$ and declare our decision by thresholding $\hat{p}_t(\vec{x}_t)$ as
	\begin{equation}
	\hat{d}_t=\text{sign}(\tau_t-\hat{p}_t(\vec{x}_t)),
	\end{equation}
	where $\tau_t$ is the threshold level. After declaring our decision, we may or may not observe the true label $d_t$ as a feedback. We use this information to optimize $\tau_t$ whenever we observe the correct decision $d_t$.
	We define the loss of thresholding $\hat{p}_t(\vec{x}_t)$ by $\tau_t$ as 
	\begin{equation} \label{eq:anomalyloss}
	l_A(\tau_t,\hat{p}_t(\vec{x}_t),d_{t})=C_{d_t}\mathbf{1}_{\lbrace \text{sign}(\tau_t-\hat{p}_t(\vec{x}_t))\neq d_t \rbrace}.
	\end{equation}
	We define the regret of choosing the threshold value $\tau_t$ against a specific threshold $\tau$ (which can even be the unknown ``best" threshold that minimizes the cumulative error) at round $t$ by
	\begin{equation}
	r_{A,t}(\tau_t,\tau)=l_A(\tau_t,\hat{p}_t(\vec{x}_t),d_{t})-l_A(\tau,\hat{p}_t(\vec{x}_t),d_{t}).
	\end{equation}
	We denote the accumulated anomaly detection regret up to time $T$ by
	\begin{equation} \label{eq:anomalyregret}
	R_{A,T}=\sum_{t=1}^{T}r_{A,t}(\tau_t,\tau).
	\end{equation}
	\par 
	We emphasize that the main challenge in ``two-step" approaches for anomaly detection is to construct a density function $\hat{p}_t(\cdot)$, which powerfully models the observations distribution. For this purpose, in section \ref{sec2}, we first introduce an algorithm, which achieves the performance of a wide range of single component density functions. Based on this algorithm, in section \ref{BT}, we use a nested tree structure to construct a multi-modal density estimation algorithm. In section \ref{anomalydetection}, we introduce our adaptive thresholding scheme, which will be used on the top of the density estimator described in section \ref{BT} to form our complete anomaly detection algorithm.
	
	\section{Single Component Density Estimation} \label{sec2}
	
	In this section we introduce an algorithm, which sequentially achieves the performance of the best single component distribution in the exponential family of distributions \cite{exfam}. At each round $t$, we observe a sample vector $\vec{x}_t \in \mathbb{R}^m$, drawn from an exponential-family distribution
	\begin{equation}\label{eq:expofam}
	f(\vec{x}_t)=h(\vec{x}_t)\exp{(\langle \eta,\vec{s}_t \rangle\ - A(\eta))},
	\end{equation}
	where
	\begin{itemize}
		\item $\eta \in \mathbf{F}$ is the unknown ``natural parameter" of the exponential-family distribution. Here, $\mathbf{F} \subset \mathbb{R}^d$ is a bounded convex set.
		\item $h(\vec{x}_t)$ is the ``base measure function" of the exponential-family distribution.
		\item $A(\eta)$ is the ``log-partition function" of the distribution.
		\item $\vec{s}_t\in \mathbb{R}^d$ is the ``sufficient statistics vector" of $\vec{x}_t$. Given the type of the exponential-family distribution, e.g., Gaussian, Bernoulli, Gamma, etc., $\vec{s}_t$ is calculated as a function of $\vec{x}_t$, i.e., $\vec{s}_t=T(\vec{x}_t)$.
	\end{itemize}
	With an abuse of notation, we put the ``base measure function" $h(\vec{x}_t)$ inside the exponential part by setting $\vec{s}_t=[\vec{s}_t;\log(h(\vec{x}_t))]$ and $\eta=[\eta;1]$. Hence, from now on, we write
	\begin{equation}
	f(\vec{x}_t)=\exp{(\langle \eta,\vec{s}_t \rangle\ - A(\eta))}.
	\end{equation}
	At each round $t$, we estimate the natural parameter $\eta$ based on the previously observed sample vectors, i.e., $\lbrace \vec{x}_1,\vec{x}_2,...,\vec{x}_{t-1} \rbrace$, and denote our estimate by $\hat{\eta_t}$. The density estimate at time $t$ is given by
	\begin{equation}
	\hat{f}_t(\vec{x}_t)=\exp{(\langle \hat{\eta}_t,\vec{s}_t \rangle\ - A(\hat{\eta}_t))}.
	\end{equation}
	In order to produce our estimate $\hat{\eta}_t$, we seek to minimize the accumulated loss we would suffer following this $\hat{\eta}_t$ during all past rounds, i.e.,
	\begin{equation}\label{eq:ML}
	\hat{\eta}_t=\argmin_{\eta}{\sum_{\tau=1}^{t-1} l({\eta},\vec{x}_{\tau})},
	\end{equation}
	where
	\begin{equation}\label{eq:loss}
	l({\eta},\vec{x}_{\tau})=-\langle \eta,\vec{s}_{\tau} \rangle\ + A(\eta).
	\end{equation}
	This is a convex optimization problem. Finding the point in which the gradient is zero, it can be seen that it suffices to choose the $\hat{\eta}_t$ such that
	\begin{equation}\label{eq:MLA}
	m_{\hat{\eta}_{t}}=\frac{\sum_{\tau=1}^{t-1} s_{\tau}}{t-1},
	\end{equation}
	where $m_{\hat{\eta}_{t}}$ is the mean of $\vec{s}_t$ when $\vec{x}_t$ is distributed with the natural parameter $\hat{\eta}_{t}$.\par 
	Note that the memory demand of our single-component density estimator does not increase with time, as is suffices to keep the sample mean of the ``sufficient statistic vectors", i.e., $s_{\tau}$'s, in memory. The complete pseudo code of our single component density estimator is provided in Alg. \ref{tab:table2}.
	
	\begin{algorithm}[t]
		\caption{Single Component Density Estimator}
		\label{tab:table2}
		\begin{algorithmic}[1]
			\algsetup{linenosize=\tiny}
			\STATE Initialize $m_s^0=0$
			\STATE Select $\hat{\eta}_1 \in \mathbf{F}$ arbitrarily
			\FOR {$t=1$ \TO $T$ }
			\STATE Observe $\vec{x}_t \in \mathbb{R}^m$
			\STATE Calculate $\vec{s}_t=T(\vec{x}_t)$
			\STATE Suffer the loss $l(\hat{\eta}_{t},\vec{x}_t)$ according to \eqref{eq:loss}
			\STATE Calculate $m_s^t=\frac{m_s^{t-1}\times(t-1)+\vec{s}_t}{t}$
			\STATE Calculate $\hat{\eta}_{t+1}$ s.t. $m_{\hat{\eta}_{t+1}}=m_s^t$
			\ENDFOR
		\end{algorithmic} 	
	\end{algorithm} 
	
\section{Multimodal Density Estimation}\label{BT}
	In this section, we extend our basic density estimation algorithm to model the observation vectors using multi-modal density functions of the form
	\begin{equation}\label{multimodalden}
	p(\vec{x}_t)=\sum_{n=1}^{N} \alpha_n f_n(\vec{x}_t),
	\end{equation}
	where each $f_n(\cdot)$ is an exponential-family density function as in \eqref{eq:expofam} and $(\alpha_1,...,\alpha_N)$ is a probability simplex, i.e., $\forall n: \alpha_n \geq 0$, $\sum_{n=1}^{N} \alpha_n=1$. \par 
	In order to construct such model, we split the space of sample vectors into several subspaces and run an independent copy of the Alg. \ref{tab:table2} in each subspace. Each one of these density estimators observe only the sample vectors, which fall into their corresponding subspace. We adaptively combine the aforementioned single component density estimators to produce our multi-modal density function. In the following, in Section \ref{sub:mixture}, we first suppose that a set of subspaces is given and explain how we combine the density estimators running over the subspaces. Then, in Section \ref{IDT}, we explain how we construct our set of subspaces using an incremental decision tree.

	\subsection{Mixture of Single Component Density Estimators}\label{sub:mixture}
	
	\begin{figure}[t]\centering
		\includegraphics[width=0.7\linewidth]{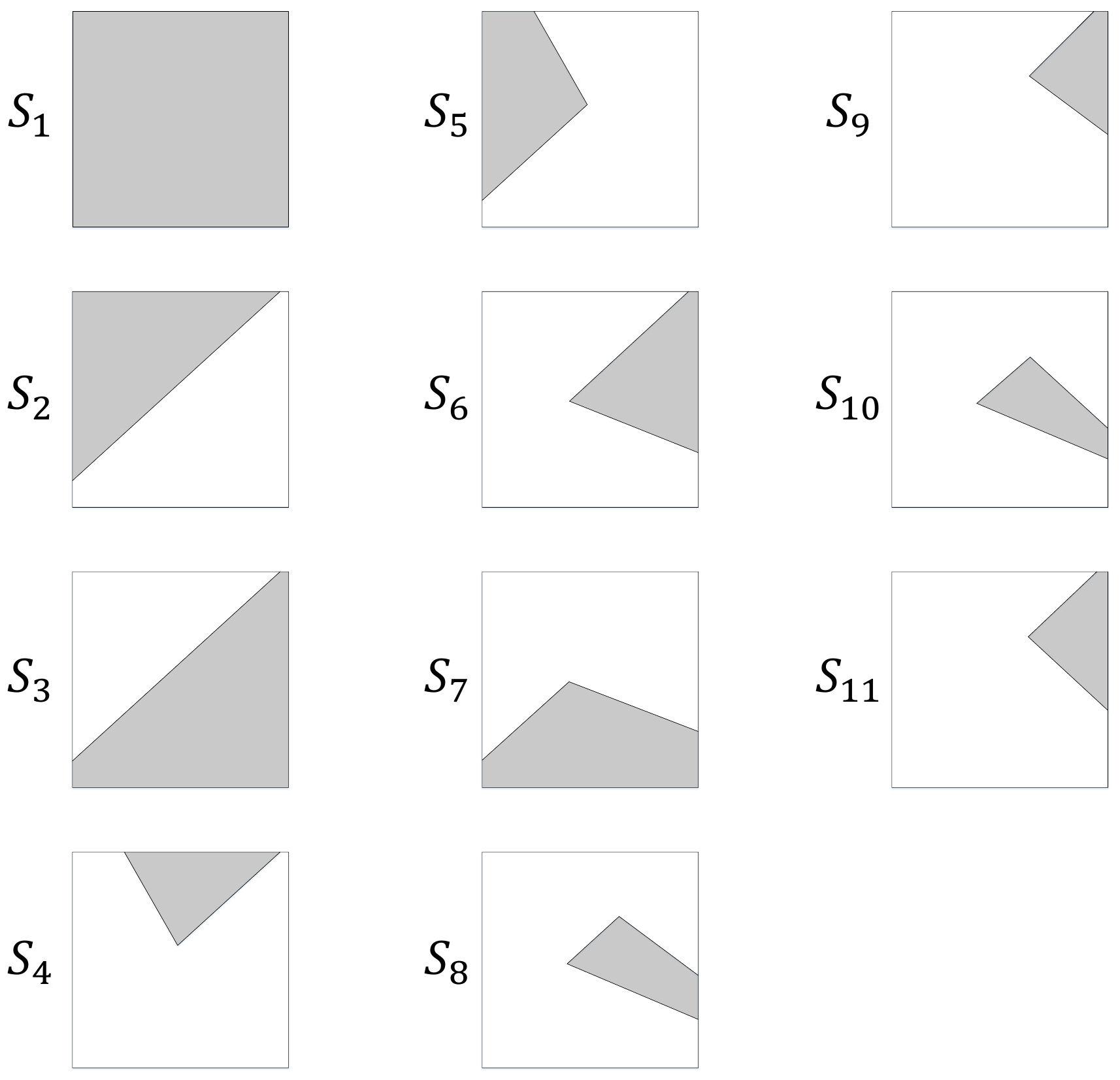}
		\caption{An example of $11$ subspaces of $\mathbb{R}^2$. The square shapes represent the whole $\mathbb{R}^2$ space and the gray regions show subspaces.} \label{fig:subspaces}
	\end{figure}

	Let $\mathcal{S}=\lbrace S_1,...,S_N \rbrace$ be a given set of $N$ subspaces of the observation space. For instance, in Fig. \ref{fig:subspaces} a set of $11$ subspaces in $\mathbb{R}^2$ is shown. We run $N$ independent copies of the Alg. \ref{tab:table2} in these subspaces and denote the estimated density function corresponding to $S_i$ at round $t$ by $\tilde{f}_{t,i}(\cdot)$. We adaptively combine $\tilde{f}_{t,i}(\cdot)$, $i=1,...N$, in a mixture-of-experts setting using the well known Exponentiated Gradient (EG) algorithm \cite{EG}. At each round $t$, we declare our multi-modal density estimation as
	\begin{equation}\label{eq:p}
	\tilde{p}_t(\cdot)=\sum_{i=1}^{N} \tilde{\alpha}_{t,i} \tilde{f}_{t,i}(\cdot),
	\end{equation}
	where $\tilde{\alpha}_{1,i}$'s are initialized to be $1/N$ for $i=1,...,N$. After observing $\vec{x}_t$, we suffer the loss $l(\tilde{p}_t,\vec{x}_t)=-\log(\tilde{p}_t(\vec{x}_t))$ and update the mixture coefficients as
	\begin{equation}\label{eq:alpha}
	\tilde{\alpha}_{t+1,i}=\tilde{\alpha}_{t,i}\exp 
	\left(
	\theta
	\frac{\tilde{f}_{t,i}(\vec{x}_t)}{\tilde{p}_t(\vec{x}_t)}
	\right),
	\end{equation}
	where $\theta$ is the learning rate parameter. The following proposition shows that in a $T$ rounds trial, we achieve a regret bound of $O(\sqrt{T})$ against the multi-modal density estimator with the best $\tilde{\alpha}$ variables (in the log-loss sense), i.e., the best convex combination of our single component density functions.
	\begin{theorem}\label{pro1}
	For a $T$ round trial, let $R$ be a bound such that $\max_{t,n}\lbrace \tilde{f}_{t,n} \rbrace \leq R,$ for all $t,n$. Let ${p}^*_t(\cdot)=\sum_{n=1}^{N} {{\alpha}^*_n} \tilde{f}_{t,n}(\cdot)$ be the optimal (in the accumulated log-loss sense) convex combination of $\tilde{f}_{t,i}$'s with fixed coefficients $\left({\alpha}^*_1,...,{\alpha}^*_N\right)$ selected in hindsight. If the accumulated log-loss of ${p}^*_t(\vec{x}_t)$ is upper bounded as
	\begin{equation}\label{eq:at}
	\sum_{t=1}^{T} l_P\left({p}^*_t(\vec{x}_t)\right)\leq AT,
	\end{equation}
	we achieve a regret bound as
	\begin{equation}\label{eq:regretp}
	R_{P,T}(\tilde{p}_t(\cdot),p^*_t(\cdot))\leq
	\sqrt{ 2AT\ln{N} }+
	\frac{R^2 \ln{N}}{2}.
	\end{equation}
	\end{theorem}

	\begin{proof}
		Denoting the relative entropy distance \cite{red} between the best probability simplex $\left({\alpha}^*_1,...,{\alpha}^*_N\right)$ and the initial point $\left( \tilde{\alpha}_{1,1},...,\tilde{\alpha}_{1,N}\right)$ by $D$, since $\tilde{\alpha}_{1,n}=1/N, \forall n=1,...,N$, we have
		\begin{equation}
		D \leq \ln{N}-H(\left(\alpha^*_1,...,\alpha^*_N\right)),
		\end{equation}
		where $H(\left(\alpha^*_1,...,\alpha^*_N\right))$ is the entropy of the best probability simplex. Since the entropy is always positive, we have $D \leq \ln{N}$. Using Exponentiated Gradient \cite{EG} algorithm with the parameter
		\begin{equation}
		\theta=\frac{2\sqrt{\ln{N}}}{R\sqrt{2AT}+R^2\sqrt{\ln{N}}},
		\end{equation}
		we achieve the regret bound in \eqref{eq:regretp}.
	\end{proof}
	\begin{remark}
	We emphasize that one can use any arbitrary density estimator in the subspaces and achieve the performance of their best convex combination using the explained adaptive combination. However, since the exponential family distribution covers a wide set of parametric distributions and closely approximates a wide range of non-parametric real life distributions, we use the density estimator in Alg. \ref{tab:table2}.
	\end{remark}

	As shown in the theorem, no matter how the set of subspaces $\mathcal{S}$ is constructed, our multi-modal density estimate in \eqref{eq:p} is competitive against the best convex combination of the density functions defined over the subspaces in $\mathcal{S}$. However, the subspaces themselves play an important role in building a proper model for arbitrary multi-modal distributions. For instance, suppose that the true underlying model is a multi-modal PDF composed of several components, which are far away from each other. If we carefully construct subspaces, such that each subspace contains only the samples generated from one of the components (or these subspaces are included in $\mathcal{S}$), then the best convex combination of the subspaces will be a good model for the true underlying PDF. This scenario is further explained through an example in Section \ref{Ex:1}. \par 
	
	In the following subsection, we introduce a decision tree approach \cite{decisiontree} to construct a growing set of proper subspaces and fit a model of the form \eqref{multimodalden} to the sample vectors. Using this tree, we start with a model with $N=1$ and increase $N$ as we observe more samples. Hence, while mitigating overfitting issues due to the $\ln{N}$ bound in \eqref{eq:regretp}, our modeling power increases with time.

	\subsection{Incremental Decision Tree}\label{IDT}

	We introduce a decision tree to partition the space of sample vectors into several subspaces. Each node of this tree corresponds to a specific subspace of the observation space. The samples inside each subspace are used to train a single component PDF. These single component probability density functions are then combined to achieve the performance of their best convex combination. \par 
	
	As explained in Section \ref{sub:mixture}, our adaptive combination of single component density functions will be competitive against their best convex combination, regardless of how we build the subspaces. However, in order to closely model arbitrary multi-modal density functions of the form \eqref{multimodalden}, we seek to find subspaces that contain only the samples from one of the components. 	Clearly, this is not always straightforward (or may not be even possible), specially if the centroids of the component densities are close to each other. To this end, we introduce an incremental decision tree \cite{decisiontree} which generates a set of subspaces so that as we observe more samples, our tree adaptively grows and produces more subspaces tuned to the underlying data. Hence, using its carefully produced subspaces, we are able to generate a multi-modal PDF that can closely model the normal data even for complex multi-modal densities, which are hard to learn with classical approaches. We next explain how we construct this incremental tree. We emphasize that we use binary trees as an example and our construction can be extended to multi branch trees in a straightforward manner.\par 
	
	\begin{figure*}[t]\centering
		\includegraphics[width=0.8\linewidth]{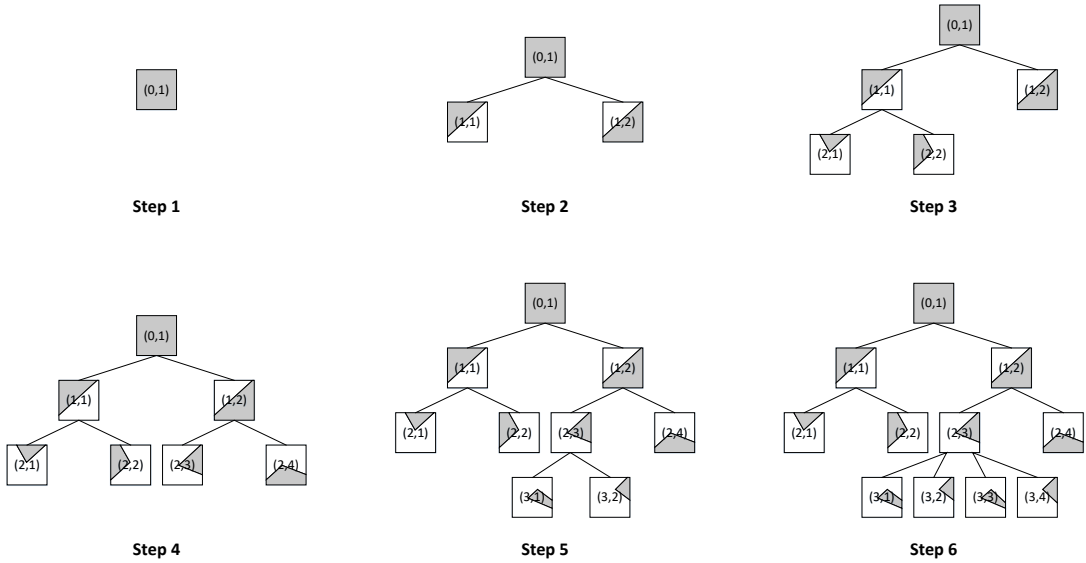}
		\caption{An example structure of the binary tree introduced in Section \ref{IDT}, where the observation space is $\mathbb{R}$ represented as squares here. The regions corresponding to the nodes are colored in gray. Each node is represented by a binary index of the form $(i,j)$, where $i$ is the level of the node, and $j$ is its order among the nodes in level $i$.}
		\label{fig:structure}
	\end{figure*}
	
	We start building our binary tree with a single node corresponding to the whole space of the sample vectors. As an example, consider step 1 in Fig. \ref{fig:structure}. We say that this node is the $1^{\text{st}}$ node in level $0$, and denote this node by a binary index of $(0,1)$, where the first element is the node's level and the second element is the order of the node among its co-level nodes. We grow the tree by splitting the subspace corresponding to a specific node into two subspaces (corresponding to two new nodes) at rounds $t={\beta}^k$ for $k=1,2,...$. Hence, at each round $t$, the tree will have $\lceil  {\log_{\beta} t} \rceil$ nodes. We emphasize that, as shown in Theorem \ref{pro1}, selecting the splitting times as the powers of $\beta$, we achieve a regret bound of $O(\sqrt{T\log{\log{T}}})$ against the best convex combination of the single component PDFs (see \eqref{eq:regretp}). Moreover, this selection of splitting times leads to a logarithmic in time computational complexity. However, again, we note that our algorithm is generic so that the splitting times can be selected in any arbitrary manner. \par 
	
	To build subspaces (or sets), we use hyperplanes to avoid overfitting. In order to choose a proper splitting hyperplane, we run a sequential $2$-means algorithm \cite{kmeans} over all the nodes as detailed in Alg. \ref{tab:tabletree}. These $2$-means algorithms are also used to select the nodes to split as follows. At each splitting time, we split the node that has the maximum ratio of ``distance between $2$ centroids" to ``$2^{ \lbrace \text{level of the node} \rbrace}$'', where ``level of the node" is the number of splits required to build the node's corresponding subspace as shown in Fig. \ref{fig:structure}. Note that as this ratio increases, it's implied that the node does not include samples from a single component PDF, which makes it a good choice to split. This motivation is illustrated using a realistic example in Section \ref{Ex:1}. We split the nodes using the hyperplane, which is perpendicular to the line connecting the two centroids of the 2-means algorithm running over the node and splits this line in half. The splitted node keeps a portion of its $\tilde{\alpha}$ value for itself and splits the remaining among its children. This portion, which is a parameter of the algorithm is denoted by $\xi$. We emphasize that using the described procedure, each node may be splitted several times. Hence, if the splitting hyperplane is not proper due to lack of observations, the problem can be fixed later by splitting the node again with more accurate hyperplanes in the future rounds. As an example, consider Fig. \ref{fig:structure}. At the last step, node $(2,3)$ is splitted again with a slightly shifted splitting line. This is illustrated in more detail using an example in Section \ref{Ex:1}. The algorithm pseudo code is provided in Alg. \ref{tab:tabletree}. 
	\begin{remark}
		We use linear separation hyperplanes to avoid overtraining while the modeling power is attained by using an incremental tree. However, our method can directly used with different separation hyperplanes.
	\end{remark}

	\begin{algorithm}[t]
		\caption{IDT-based Multi-modal Density Estimator}
		\label{tab:tabletree}
		\begin{algorithmic}[1]
			\algsetup{linenosize=\tiny}
			\STATE Select parameters $\beta$ and $\xi$
			\STATE Initialize $N=1$
			\STATE Initialize $\Sigma x_{1,L}=\Sigma x_{1,R}=0$ (zero vector)
			\STATE Initialize $\xi_{1,L}=\xi_{1,R}=1$
			\STATE Run Alg. \ref{tab:table2} over node $1$.
			\FOR{$t=1$ \TO $T$}
			\STATE Declare $\tilde{p}_t(\cdot)$ as \eqref{eq:p}
			\STATE Observe $\vec{x}_t$
			\FOR{$n=1$ \TO $N$}
			\IF{$\vec{x}_t \in r_n$}
			\STATE Update $\tilde{f}_{t,n}$ using Alg. \ref{tab:table2}.
			\IF{$\|\frac{\Sigma x_{n,L}}{\xi_{n,L}}-\vec{x}_t \| \leq \|\frac{\Sigma x_{n,R}}{\xi_{n,R}}-\vec{x}_t \|$}
			\STATE $\Sigma x_{n,L}=\Sigma x_{n,L}+\vec{x}_t$
			\STATE $\xi_{n,L}=\xi_{n,L}+1$
			\ELSE
			\STATE $\Sigma x_{n,R}=\Sigma x_{n,R}+\vec{x}_t$
			\STATE $\xi_{n,R}=\xi_{n,R}+1$
			\ENDIF
			\ENDIF
			\ENDFOR
			\STATE Update $\tilde{\alpha}$ variables as \eqref{eq:alpha}
			\IF{$t=\beta^k$}
			\STATE Select the node $n$ as explained in section \ref{IDT}
			\STATE Let $L=\Sigma x_{n,L}/\xi_{n,L}$, $R=\Sigma x_{n,R}/\xi_{n,R}$
			\STATE Split the node using the hyperplane with normal vector of $a=A/\|A\|$ and $b=<A,(L+R)/2>$, where $A=L-R$. (Hyperplane: $<x,a>=b$)
			\STATE Run copies of Alg. \ref{tab:table2} over new nodes.
			\ENDIF
			\ENDFOR
		\end{algorithmic} 	
	\end{algorithm} 
	As detailed in Alg. \ref{tab:tabletree}, at each round $t$, the tree nodes declare their single component PDFs, i.e., $\tilde{f}_{t,i}(\cdot), i=1,...,N$. We combine these density functions using \eqref{eq:p} to produce our multi-modal density estimate $\tilde{p}_t(\cdot)$. Then, the new sample vector $\vec{x}_t$ is observed and we suffer our loss as \eqref{eq:logloss}. Subsequently, we update the combination variables, i.e., $\tilde{\alpha}_{t,i}, i=1,...,N$, using \eqref{eq:alpha}. The centroids of the 2-means algorithms running over nodes are also updated as detailed in Alg. \ref{tab:tabletree}. Finally, the single component density estimates at the nodes are updated as detailed in Alg. \ref{tab:table2}. At the end of the round, if $t=\beta^k$, we update the tree structure and construct new nodes as explained in Section \ref{IDT}.\par

	In the following section, we explain our adaptive thresholding scheme, which will be used on top of described multi-modal density estimator to form our two-step anomaly detection algorithm.
	
	\section{Anomaly Detection Using Adaptive Thresholding}\label{anomalydetection}
	We construct an algorithm, which thresholds the estimated density function $\hat{p}_t(\vec{x}_t)$ to label the sample vectors. To this end, we label the sample $\vec{x}_t$ by comparing $\hat{p}_t(\vec{x}_t)$ with a threshold $\tau_t$ as
	\begin{equation}\label{eq:anomalydec}
	\hat{d}_t=
	\begin{cases}
	+1, &\hat{p}_t(\vec{x}_t)<\tau_t\\
	-1, &\hat{p}_t(\vec{x}_t)\geq\tau_t.
	\end{cases}
	\end{equation}
	Suppose at some specific rounds $t \in T_f$, after we declared our decision $\hat{d}_t$ the true label $d_t$ is revealed. We seek to use this information to minimize the total regret defined in \eqref{eq:anomalyregret}. However, since we observe the incurred loss only at rounds $t \in T_f$, we restrict ourselves to these rounds. Moreover, since the loss function used in \eqref{eq:anomalyregret} is based on the indicator function that is not differentiable, we substitute the loss function defined in \eqref{eq:anomalyloss} with the well known logistic loss function defined as
	\begin{equation}\label{eq:tildeloss}
	\tilde{l}(\tau_t,\hat{p}_t(\vec{x}_t),d_{t})=C_{d_t}\log(\exp((\hat{p}_t(\vec{x}_t)-\tau_t)d_t)+1).
	\end{equation}
	Our aim is to achieve the performance of the best constant $\tau$ in a convex feasible set $\mathbf{G}$. To this end, we define our regret as
	\begin{equation}\label{eq:pseudoreg}
	\tilde{R}_{T_f}=\sum_{t\in T_f}^{} \tilde{l}(\tau_t,\hat{p}_t(\vec{x}_t),d_{t})-\min\limits_{\tau \in \mathbf{G}}\sum_{t\in T_f}^{}\tilde{l}(\tau,\hat{p}_t(\vec{x}_t),d_{t}),
	\end{equation}
	We use the Online Gradient Descent algorithm \cite{Hazan2007} to produce our threshold level $\tau_t$. To this end, we choose ${\tau}_1 \in \mathbf{G}$ arbitrarily. At each round $t$, after declaring our decision $\hat{d}_t$, we construct
	\begin{equation}\label{eq:threshold}
	\tau_{t+1}=
	\begin{cases}
	\mathbb{P}_{\mathbf{G}}\left(\tau_{t}- \alpha_t \nabla_\tau \tilde{l}(\tau_t,\hat{p}_t(\vec{x}_t),d_{t})\right), &\text{if $d_t$ is known} \\
	\tau_{t}, &\text{otherwise},
	\end{cases}
	\end{equation}
	where $\alpha_t$ is the step size at time $t$ and $\mathbb{P}_{\mathbf{G}}(\cdot)$ is a projection function defined as
	\begin{equation}\label{eq:proj}
	\mathbb{P}_{\mathbf{G}}(a)=\argmin_{b \in \mathbf{G}} \| b-a \|.
	\end{equation}
	The complete algorithm is provided in Alg. \ref{tab:table_anomaly}. \par 
	
	For the sake of notational simplicity, from now on, we assume that $d_t$ is revealed at all time steps. We emphasize that since the rounds with no feedback do not affect neither the threshold in \eqref{eq:threshold}, nor the regret in \eqref{eq:pseudoreg}, we can simply ignore them in our analysis. The following theorem shows that using Alg. \ref{tab:table_anomaly}, we achieve a regret upper bound of $O(\log T)$, against the best fixed threshold level selected in hindsight.
		\begin{algorithm}[t]
		\caption{IDT-based Anomaly Detector}
		\label{tab:table_anomaly}
		\begin{algorithmic}[1]
			\algsetup{linenosize=\tiny}
			\STATE Select parameters $C_1$ and $C_{-1}$
			\STATE Fix $\alpha_t$ using \eqref{eq:stepsize} for $t=1,...,T$
			\STATE Select $\tau_1 \in \mathbf{G}$ arbitrarily
			\FOR {$t=1$ \TO $T$ }
			\STATE Observe $\hat{p}_t(\vec{x}_t)$
			\STATE Calculate $\hat{d}_t$ using \eqref{eq:anomalydec}
			\STATE Observe $d_t$
			\STATE Suffer the loss $\tilde{l}(\hat{\eta}_{t},\vec{x}_t)$ according to \eqref{eq:tildeloss}
			\STATE Calculate $\tau_{t+1}=\mathbb{P}_{\mathbf{G}}\left(\tau_{t}+ \frac{\alpha_t d_t C_{d_t}}{1+\exp((\tau_t-\hat{p}_t(\vec{x}_t))d_t)} \right)$
			\ENDFOR
		\end{algorithmic} 	
	\end{algorithm} 
	
	\begin{theorem}
		Using Alg. \ref{tab:table_anomaly} with step size
		\begin{equation}\label{eq:stepsize}
		\alpha_t=\frac{(1+\exp(\mathcal{D}_\mathbf{G}))^2}{t C_{\min} \exp (\mathcal{D}_\mathbf{G})},
		\end{equation}
		our anomaly detection regret in \eqref{eq:pseudoreg} is upper bounded as
		\begin{equation}\label{eq:anregb}
		\tilde{R}_{T}\leq \frac{ \exp(\mathcal{D}_\mathbf{G}) C^2_{\max} }{2 C_{\min}} (1+\log T),
		\end{equation} 
		where $\mathcal{D}_\mathbf{G}=\max\limits_{a,b \in \mathbf{G}} \| a-b \|$ is the diameter of the feasible set $\mathbf{G}$ including $\tau_t$ and $\hat{p}(\vec{x}_t)$. $C_{\max}$ and $C_{\min}$ are the maximum and minimum of $\lbrace C_{1},C_{-1} \rbrace$, respectively.
	\end{theorem}
	
	\begin{proof} 
		Considering the loss function in \eqref{eq:tildeloss}, we take the first derivatives of $\tilde{l}$ as
		\begin{equation}
		\frac{\partial \tilde{l}(\tau_t,\hat{p}_t(\vec{x}_t),d_t)}{\partial \tau_t}=\frac{-d_t C_{d_t}}{1+\exp(  (\tau_t-\hat{p}_t(\vec{x}_t)) d_t )},
		\end{equation}
		and its second derivative as
		\begin{equation}
		\frac{\partial^2 \tilde{l}(\tau_t,\hat{p}_t(\vec{x}_t),d_t)}{\partial \tau^2_t}=\frac{ C_{d_t}\exp(  (\tau_t-\hat{p}_t(\vec{x}_t)) d_t ) }{ (1+\exp(  (\tau_t-\hat{p}_t(\vec{x}_t)) d_t ))^2 }.
		\end{equation}
		The first derivative can be bounded as
		\begin{equation}\label{eq:g}
		|\frac{\partial \tilde{l}(\tau_t,\hat{p}_t(\vec{x}_t),d_t)}{\partial \tau_t}| \leq \frac{C_{\max}}{1+\exp(-\mathcal{D}_\mathbf{G})}.
		\end{equation}
		Similarly, the second derivative is bounded as
		\begin{equation}\label{eq:h}
		|\frac{\partial^2 \tilde{l}(\tau_t,\hat{p}_t(\vec{x}_t),d_t)}{\partial \tau^2_t}| \geq \frac{C_{\min} \exp(\mathcal{D}_\mathbf{G})}{(1+\exp(\mathcal{D}_\mathbf{G}))^2}.
		\end{equation}
		Using Online Gradient Descent \cite{Hazan2007}, with step size given in \eqref{eq:stepsize} we achieve the regret upper bound in \eqref{eq:anregb}.
	\end{proof}\par

	\section{Experiments}\label{simulations}
	In this section, we demonstrate the performance of our algorithm in different scenarios involving both real and synthetic data. In the first experiment, we have created a synthetic scenario to illustrate how our algorithm works. In this scenario, we sequentially observe samples drawn from a $4$-component distribution, where the probability density function is a convex combination of $4$ multivariate Gaussian distributions. The samples generated from one of the components are considered anomalous. The objective is to detect these anomalous samples. In the second experiment, we have shown the superior performance of our algorithm with respect to the state-of-the-art methods on a synthetic dataset, where the underlying PDF cannot be modeled as a multi-modal Gaussian distribution. The third experiment shows the performance of the algorithms on a real multi-class dataset. In this experiment, the objective is to detect the samples belonging to one specific class, which are considered anomalous. \par
	 
	We compare the density estimation performance of our algorithm \textit{ITAN}, against a set of state-of-the-art opponents composed of \textit{wGMM} \cite{silverman}, \textit{wKDE} \cite{silverman},  and  \textit{ML} algorithms. The \textit{wGMM} \cite{wgmm} is an algorithm which uses a sliding window of the last $\log{t}$ normal samples to train a GMM using the well known Expectation-Maximization (EM) \cite{wgmm} method. The length of sliding window is set to $\log{t}$ in order to have a fair comparison against our algorithm in the sense of computational complexity. In favor of the \textit{wGMM} algorithm, we provide to it the number of components that provides the best performance for that algorithm. The \textit{wKDE} is the well-known KDE \cite{silverman} algorithm that uses a sliding window of the last $\sqrt{t}$ normal samples to produce its estimate on the density function. The length of sliding window is $\sqrt{t}$ in favor of the \textit{wKDE} algorithm to produce competitive results. The kernel bandwidth parameters are chosen based on Silverman's rule \cite{silverman}. Finally, \textit{ML} algorithm is the basic Maximum Likelihood algorithm which fits the best single-component density function to the normal samples. We use our algorithm \textit{ITAN} with the parameters $\beta=2$ and $\xi=0.8$ in the all three experiments. We emphasize that no optimization has been performed to tune these parameters to the datasets. \par 
	In order to compare the anomaly detection performance of the algorithms, we use the same thresholding scheme described in Alg. \ref{tab:table_anomaly} for all algorithms. We use the ROC curve as our performance metric. Given a pair of false negative and false positive costs, denoted by $C_1$ and $C_{-1}$, respectively, each algorithm achieves a pair of True Positive Rate (TPR) and False Positive Rate (FPR), which determines a single point on its corresponding ROC curve. In order to plot the ROC curves, we have repeated the experiments $100$ times, where $C_1=1$ and $C_2$ is selected from the set of $\lbrace {\frac{i}{100}}| i=0,1,...,99\rbrace$. The ROC curves are plotted using the resulting $100$ samples. The Area Under Curve (AUC) of the algorithms are also calculated using these samples as another performance metric.

	\subsection{Synthetic Multi-modal Distribution}\label{Ex:1}

	In the first experiment, we have created $10$ datasets of length $1000$ and compared the performance of the algorithms in both density estimation and anomaly detection tasks. Each sample is labeled as ``normal" or ``anomalous" with probabilities of $0.9$ and $0.1$, respectively. The normal samples are randomly generated using the density function

	\begin{align}
	f_{\text{normal}}(\vec{x}_t)=\frac{1}{3} \Bigg( 
	&N\left( 
	\left[ {\begin{array}{cc}
		-1 \\
		1 \\
		\end{array} }\right] ,
	\left[ {\begin{array}{cc}
		0.2 & 0 \\
		0 & 0.2 \\
		\end{array} } \right] \right)\nonumber\\
	+
	&N\left( 
	\left[ {\begin{array}{cc}
		1 \\
		-1 \\
		\end{array} }\right] ,
	\left[ {\begin{array}{cc}
		0.14 & 0.2 \\
		0.2 & 0.4 \\
		\end{array} } \right] \right)\nonumber\\
	+
	&N\left( 
	\left[ {\begin{array}{cc}
		2 \\
		2 \\
		\end{array} }\right] ,
	\left[ {\begin{array}{cc}
		0.4 & -0.2 \\
		-0.2 & 0.14 \\
		\end{array} } \right] \right) \Bigg),\label{eq:normal}
	\end{align}
	while the anomalous samples are generated using
	\begin{equation}
	f_{\text{anomaly}}(\vec{x}_t)= 
	N\left( 
	\left[ {\begin{array}{cc}
		1 \\
		1 \\
		\end{array} }\right] ,
	\left[ {\begin{array}{cc}
		0.1 & 0 \\
		0 & 0.1 \\
		\end{array} } \right] \right).\label{eq:anomalous}
	\end{equation}
	Fig. \ref{fig:dataset1} shows the samples in one of the datasets used in this experiment to provide a clear visualization.\par
	
	\begin{figure}[t]\centering
		\includegraphics[width=0.7\linewidth]{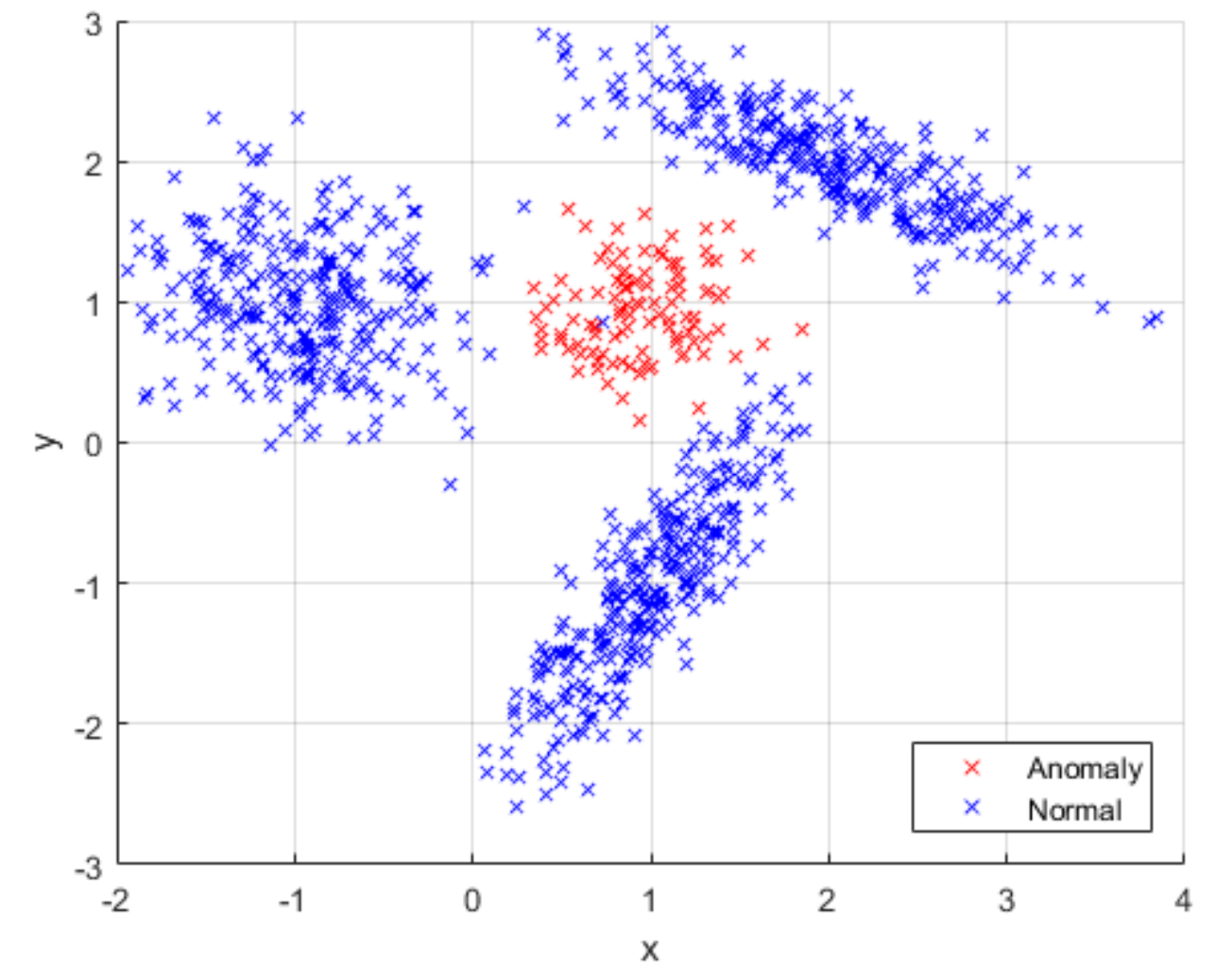}
		\caption{Visualization of samples in one of the datasets used in Experiment \ref{Ex:1}.}
		\label{fig:dataset1}
	\end{figure}

	\begin{figure*}[t]
		\centering
		\begin{adjustbox}{minipage=\linewidth,scale=0.8}
		\begin{subfigure}{0.3\textwidth}
			\includegraphics[width=\linewidth]{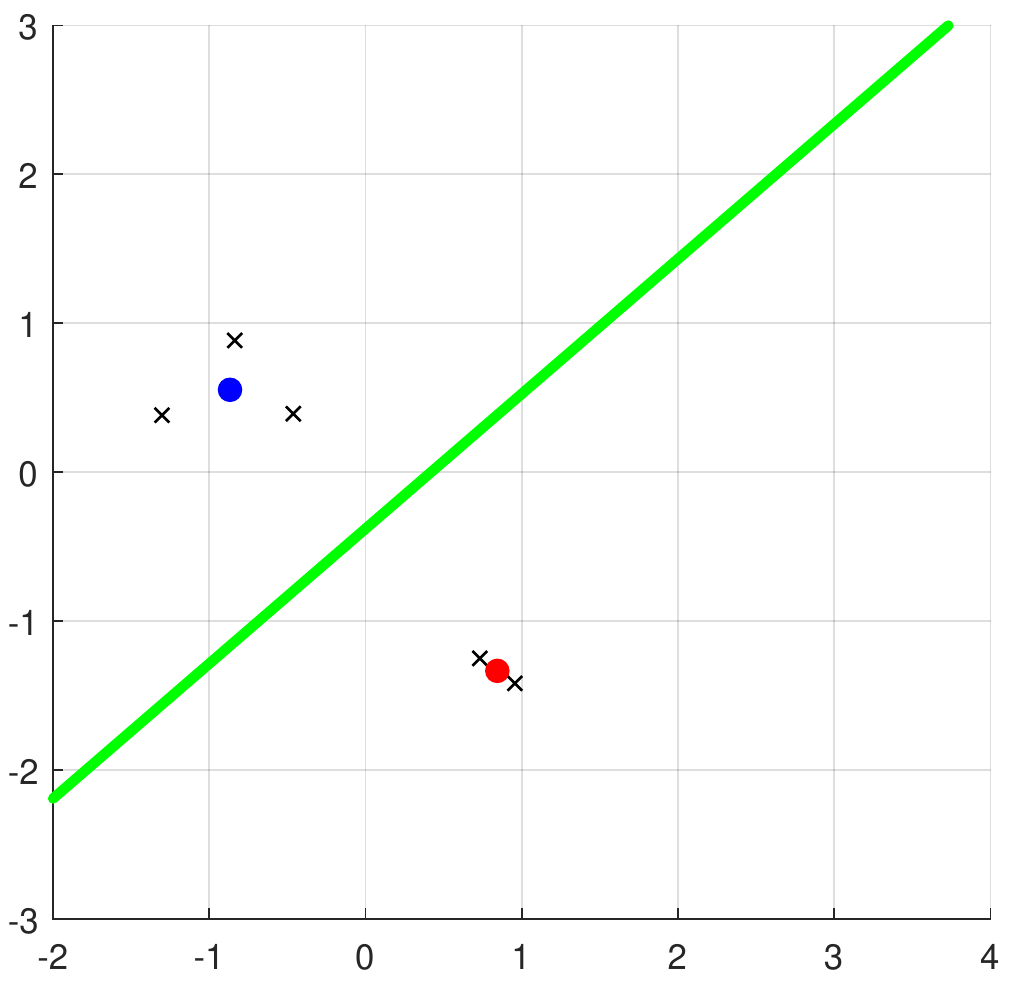}
			\caption{The samples observed until round $5$ and the first split based on these observations.} \label{fig:ex1split1}
		\end{subfigure}
		\hspace*{\fill} 
		\begin{subfigure}{0.3\textwidth}
			\includegraphics[width=\linewidth]{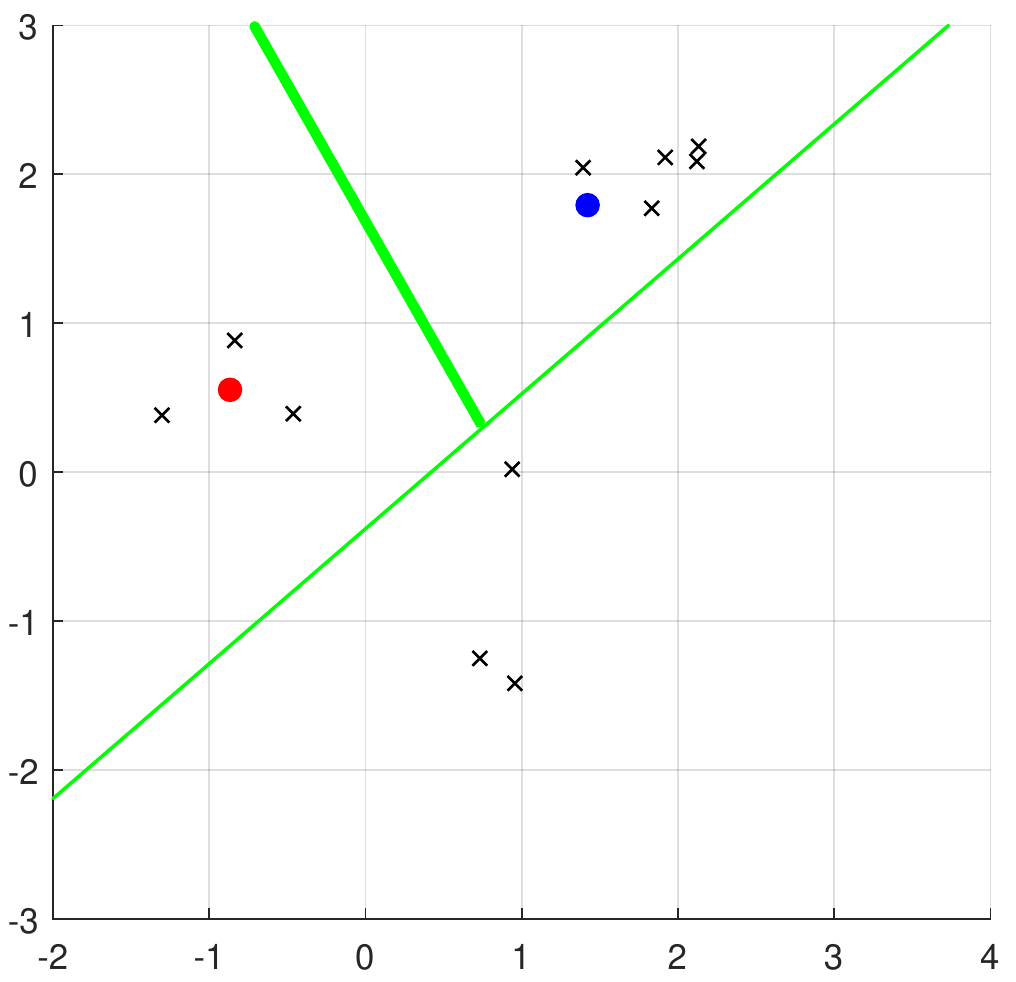}
			\caption{The samples observed until round $11$ and the second split based on these observations.} \label{fig:ex1split2}
		\end{subfigure}
		\hspace*{\fill} 
		\begin{subfigure}{0.3\textwidth}
			\includegraphics[width=\linewidth]{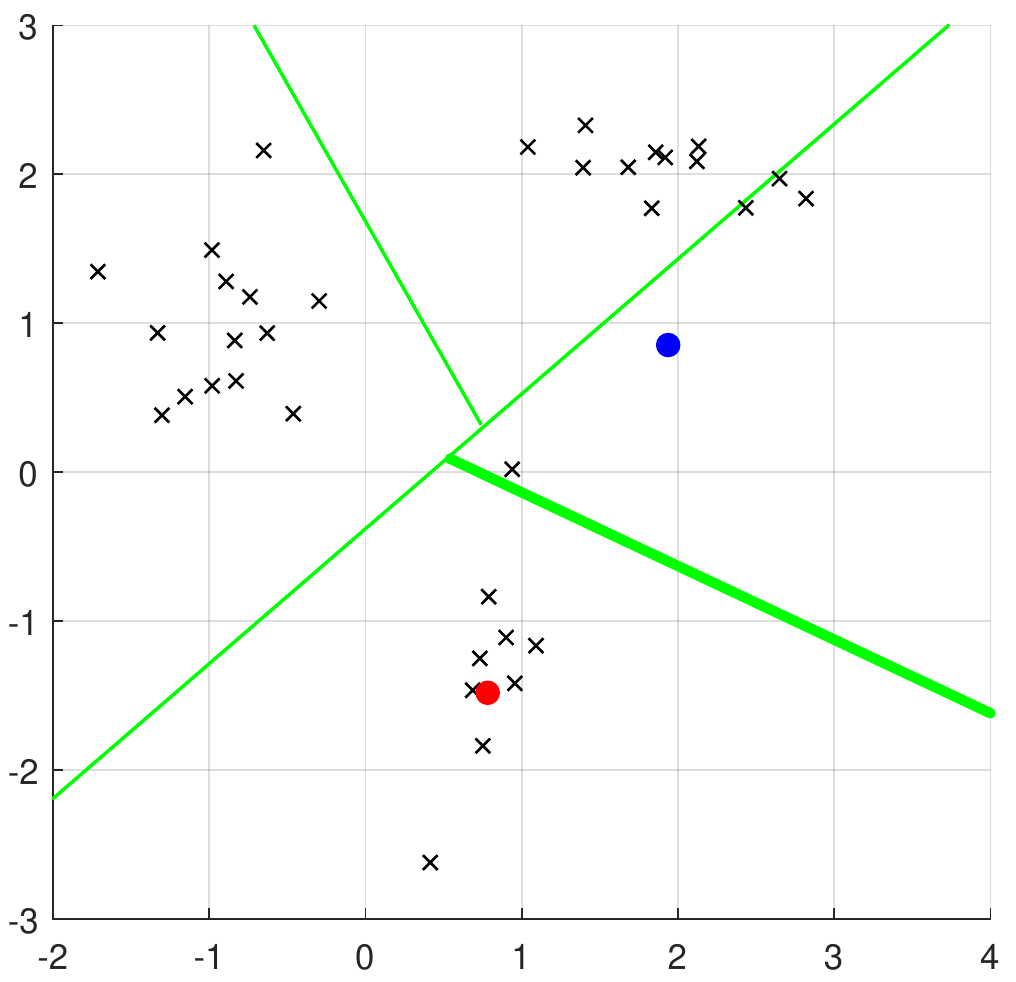}
			\caption{The samples observed until round $40$ and the third split based on these observations.}
			\label{fig:ex1split3}
		\end{subfigure}
		\hspace*{\fill} 
		\begin{subfigure}{0.3\textwidth}
			\includegraphics[width=\linewidth]{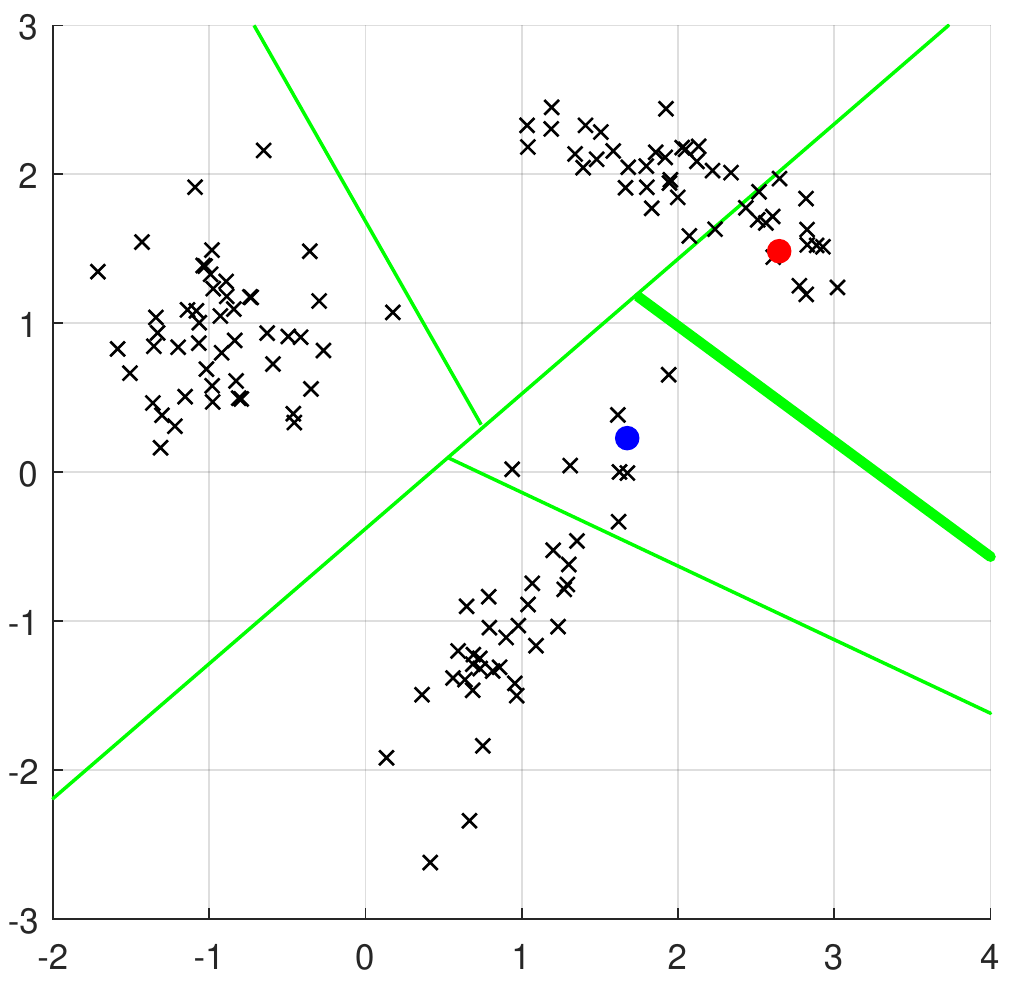}
			\caption{The samples observed until round $150$ and the fourth split based on these observations.} \label{fig:ex1split4}
		\end{subfigure}
		\hspace*{\fill} 
		\begin{subfigure}{0.3\textwidth}
			\includegraphics[width=\linewidth]{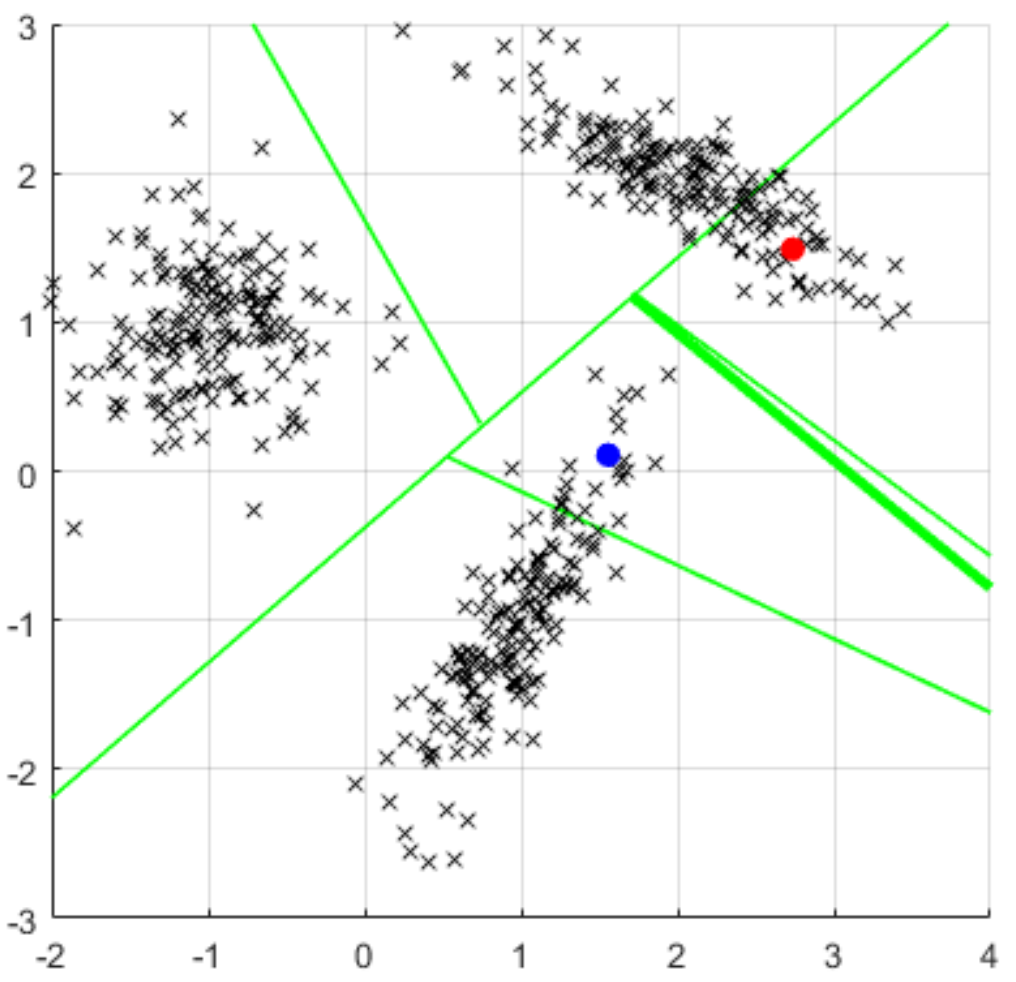}
			\caption{The samples observed until round $570$ and the fifth split based on these observations.} \label{fig:ex1split5}
		\end{subfigure}
		\hspace*{\fill} 
	\end{adjustbox}	
		\caption{An example on how the tree learns the underlying distribution of the samples. The normal samples are from a synthetic dataset generated using \eqref{eq:normal}.} 
		\label{fig:splitting}
		
	\end{figure*}
	
	In order to show how our algorithm learns, we illustrate how the tree splits the observation space, how the density estimations train their single component PDFs and how the combination of single component PDFs models the normal data in the experiment over one of the $10$ datasets. Fig. \ref{fig:splitting} shows five growth steps of the tree. In each subfigure, the observed samples are shown using black cross signs. The centroids of the 2-means algorithm running over the node that is going to split are shown using two blue and red points. The thicker green line is the new splitting line, while the thiner green lines show previous splitting lines. The splittings shown in this figure result in a tree structure that is shown in Fig. \ref{fig:structure}. Fig. \ref{fig:combination} shows how the single component PDFs defined over nodes are combined to construct our multi-modal density function. In Fig. \ref{fig:truem} the contour plot of the normal data distribution function is shown. Fig. \ref{fig:ex1single} shows the structure of the tree at the end of the experiment, the contour plots of the single component PDFs learned over the nodes, and their coefficient in the convex combination which yields the final multi-modal density function. The contour plot of this final multi-modal PDF is shown in Fig. \ref{fig:ex1multi}. As shown in these figures, the three components of the underlying PDF are almost captured by the three nodes, which are generated in the second level of our tree.

	\begin{figure*}[t]
		\centering
		\begin{adjustbox}{minipage=\linewidth,scale=1}
		\begin{subfigure}{0.28\textwidth}
			\begin{subfigure}{1\textwidth}
				\includegraphics[width=\linewidth]{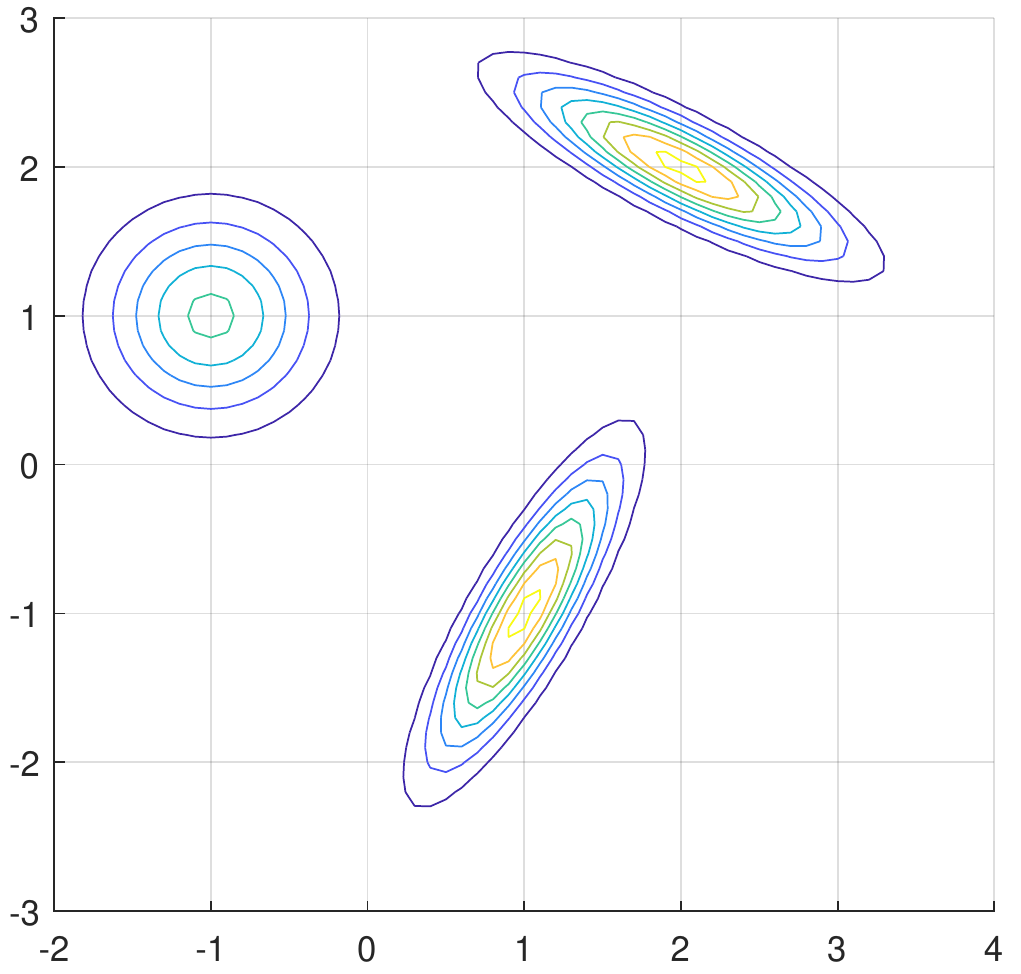}
				\caption{The contour plot of the normal data distribution defined in \eqref{eq:normal}.} \label{fig:truem}
			\end{subfigure}
			
			\begin{subfigure}{1\textwidth}\vspace{1.4cm}
				\centering
				\includegraphics[width=\textwidth]{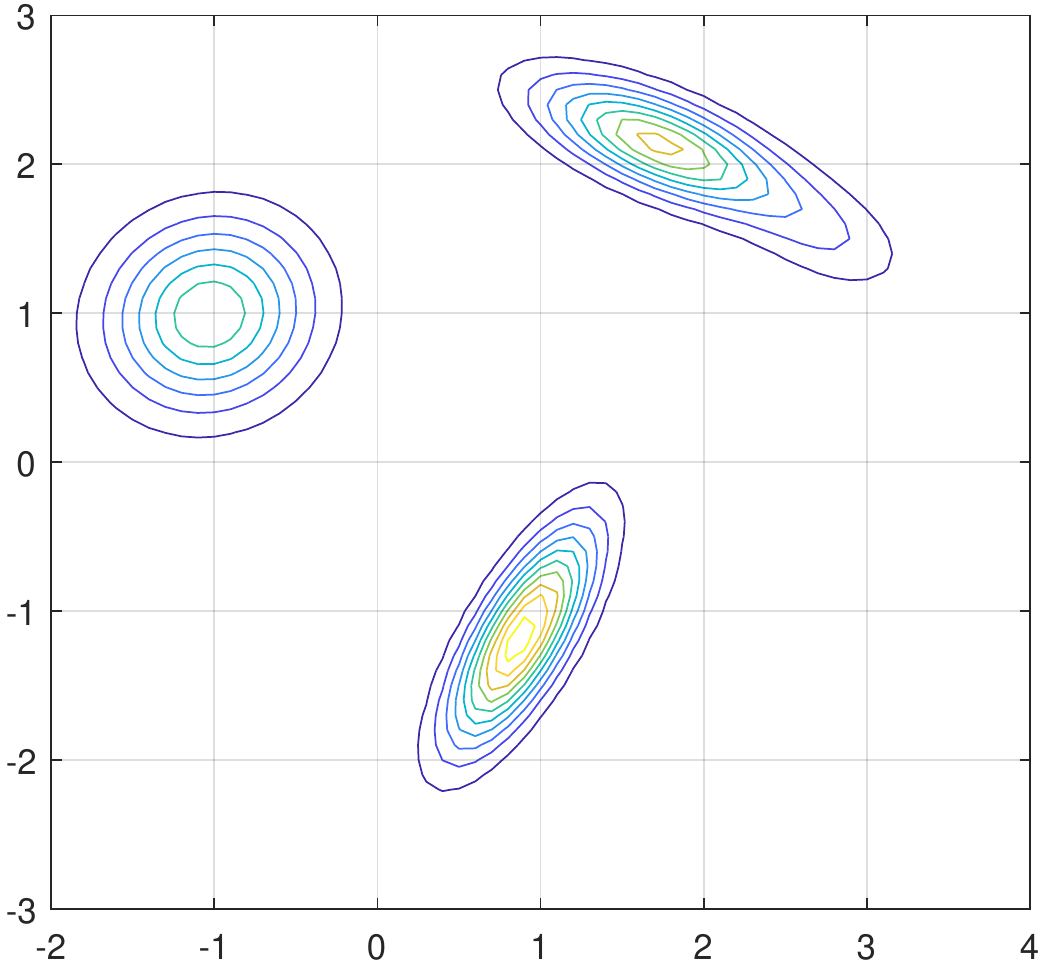}\\
				\caption{The contour plot of the normal data which is learned by the tree at the end of the experiment.}
				\label{fig:ex1multi}
			\end{subfigure}
		\end{subfigure}
		\hspace*{\fill} 
		\begin{subfigure}{0.68\textwidth}
			\includegraphics[width=\textwidth]{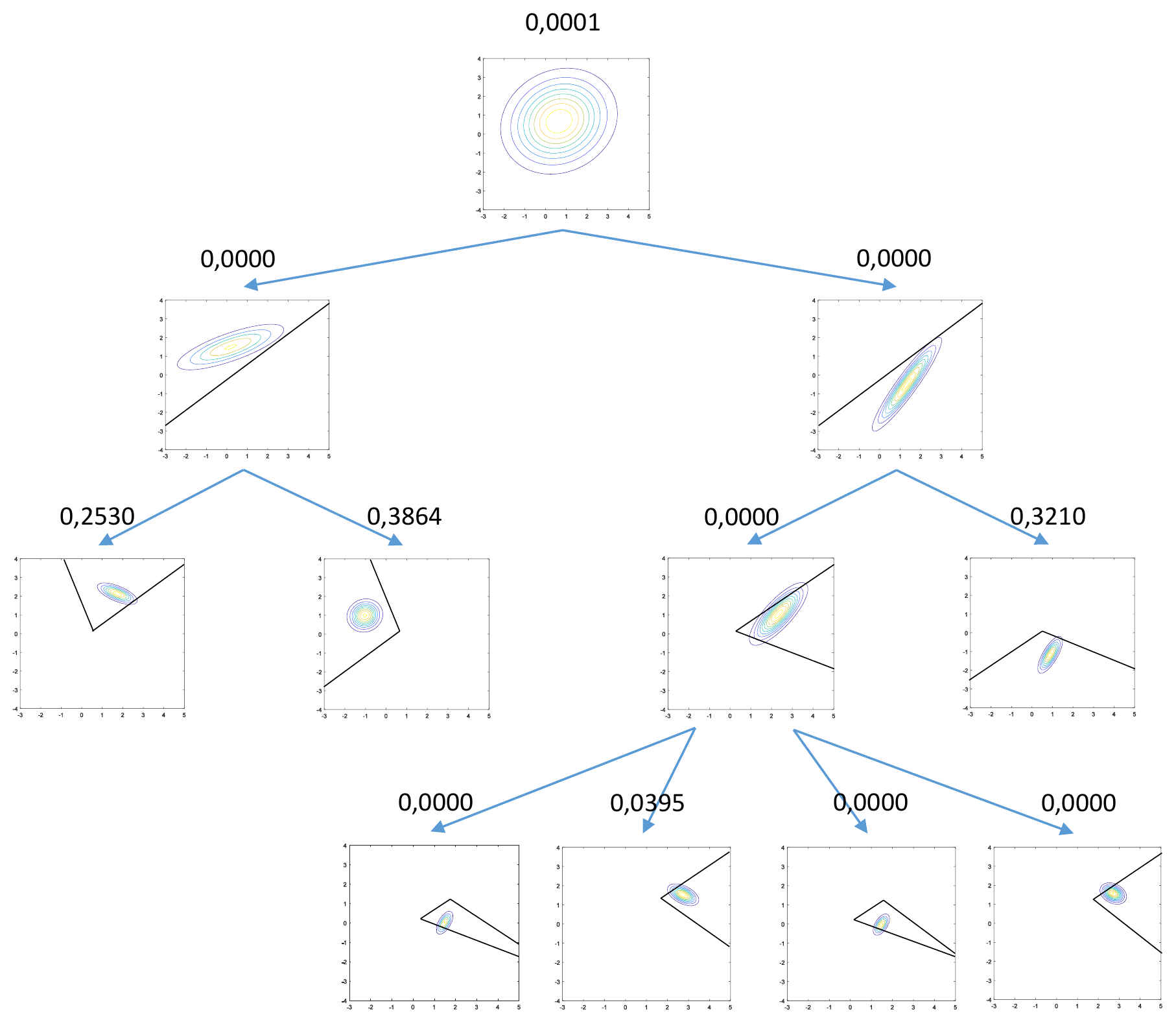}\\
			\caption{The tree structure and the contour plots of the single component density functions constructed over the nodes at the end of experiment. The coefficients of the nodes in final convex combination are shown above the nodes. The shown $11$ PDFs over nodes combined with their corresponding coefficients result in the PDF shown in \ref{fig:ex1multi}.}
			\label{fig:ex1single}
		\end{subfigure}
		\hspace*{\fill} 
		\caption{The true underlying PDF, the tree structure and the single component PDFs defined over nodes, and the final PDF learned by the algorithm at the end of the experiment on one of the datasets of the first experiment described in Section \ref{Ex:1}.} 
		\label{fig:combination}
		\end{adjustbox}
	\end{figure*}

	In order to compare the density estimation performance of the algorithms, their averaged loss per round defined by $\text{Loss}(t)=\sum_{\tau=1}^{t}l_P(\hat{p}_{\tau}(\vec{x}_{\tau}))/t$, are shown in Fig. \ref{fig:ex1pdf}. The loss of all algorithms on the rounds with anomalous observations are considered as $0$ in these plots. The anomaly performance of the algorithms are compared in Fig. \ref{fig:ex1roc}. This figure shows the ROC curves of the algorithms averaged over 10 datasets. The averaged log-loss performance, AUC results and running time of the algorithms are provided in Table. \ref{tab:ex1}. All results are obtained using a Intel(R) Core(TM) i5-4570 CPU with 3.20 GHz clock rate and 8 GB RAM.
	
	\begin{figure}\centering
		\includegraphics[width=0.7\linewidth]{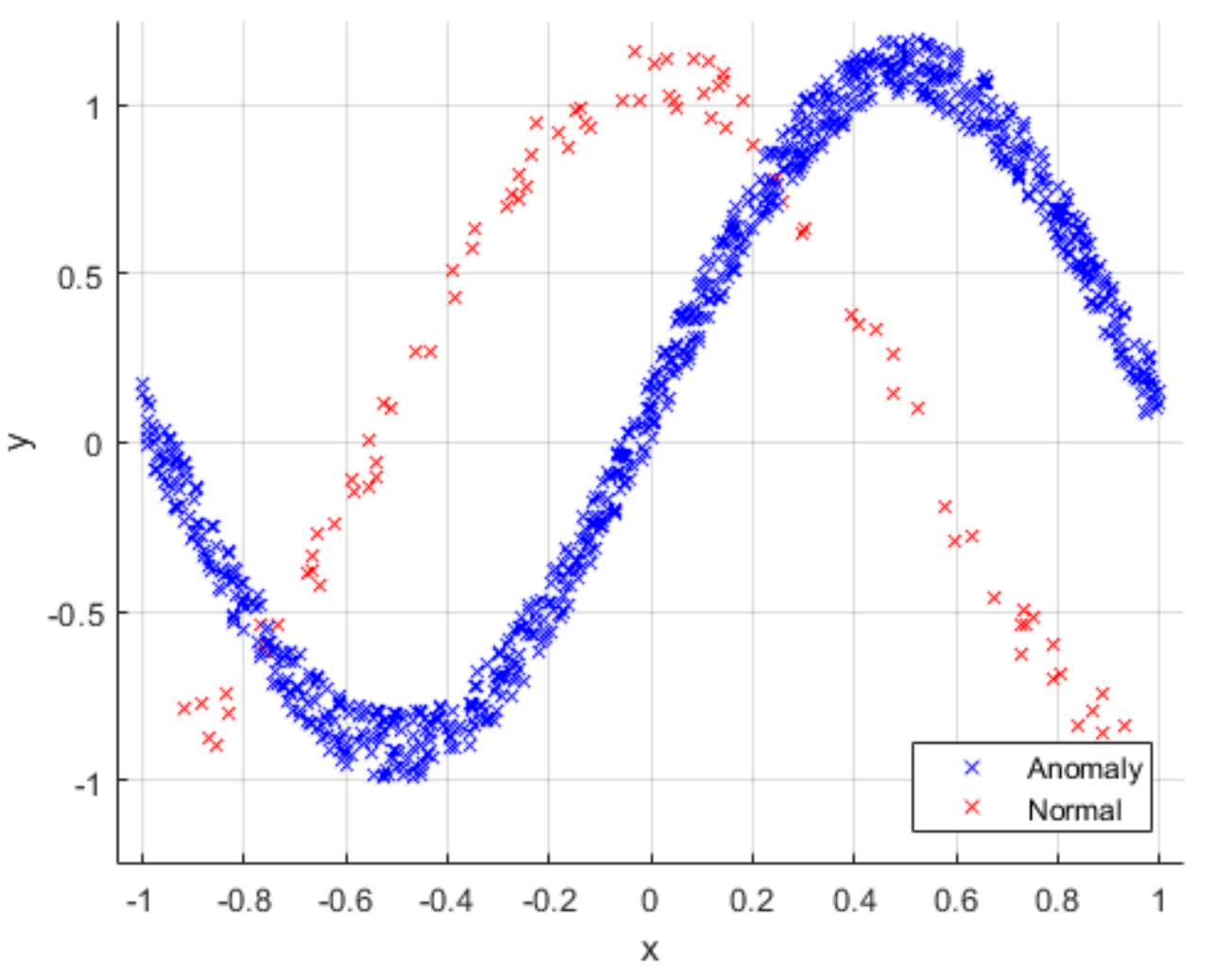}
		\caption{Visualization of samples in one of the datasets used in Experiment \ref{Ex:2}.}
		\label{fig:dataset2}
	\end{figure}
	
	\begin{table}[t]
		\centering
		\caption{``Log-Loss", ``AUC" and ``Running Time" of the algorithms over the datasets described in Section \ref{Ex:1}. The AUC and Running Time values are in the format of ``mean value $\pm$ standard deviation. }
		\label{tab:ex1}
		\begin{tabular}{ |p{1.2cm}||p{1.1cm}|p{2.2cm}|p{2.4cm}|  }
			\hline
			Algorithm& Log-Loss & Area Under Curve  & Running Time (ms)\\
			\hline
			ITAN  & $2.174$& $0.8281 \pm 0.1383$    & $313.02 \pm 5.82$ \\
			wGMM   & $3.056$&$0.8394 \pm 0.0210$ & $2275.67 \pm 123.23$\\
			wKDE & $3.022$&   $0.5678 \pm 0.0228$  & $273.16 \pm 6.92$\\
			ML & $3.087$&   $0.2700 \pm 0.0943$  & $38.56 \pm 1.33$\\
			\hline
		\end{tabular}
	\end{table}

	\begin{figure*}[t]
		\hspace*{\fill} 
		\begin{subfigure}{0.3\textwidth}
			\includegraphics[width=\linewidth]{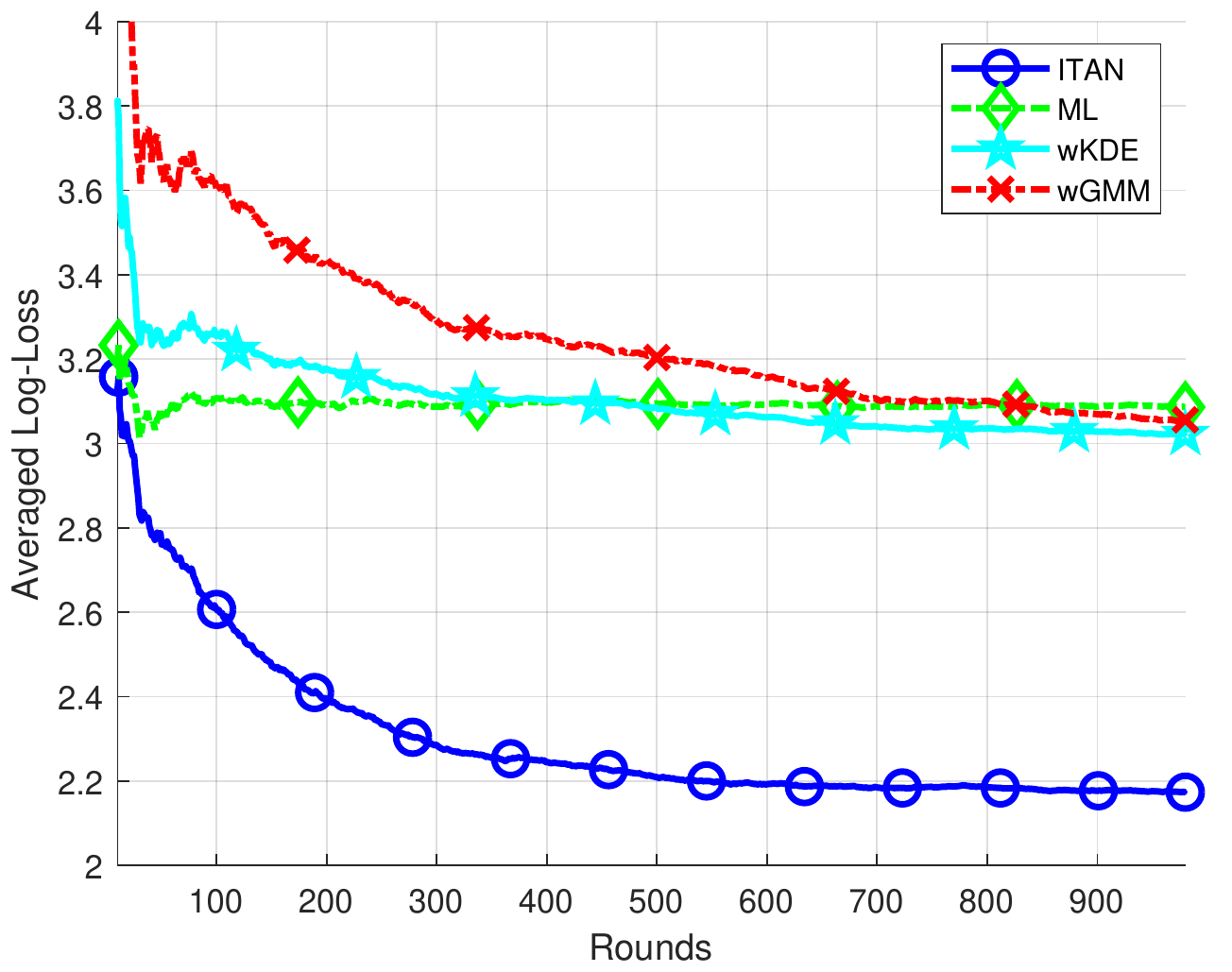}
			\caption{Density estimation loss of the algorithms over the datasets described in Section \ref{Ex:1}. The results are averaged over $10$ datasets.} \label{fig:ex1pdf}
		\end{subfigure}
		\hspace*{\fill} 
		\begin{subfigure}{0.3\textwidth}
			\includegraphics[width=\linewidth]{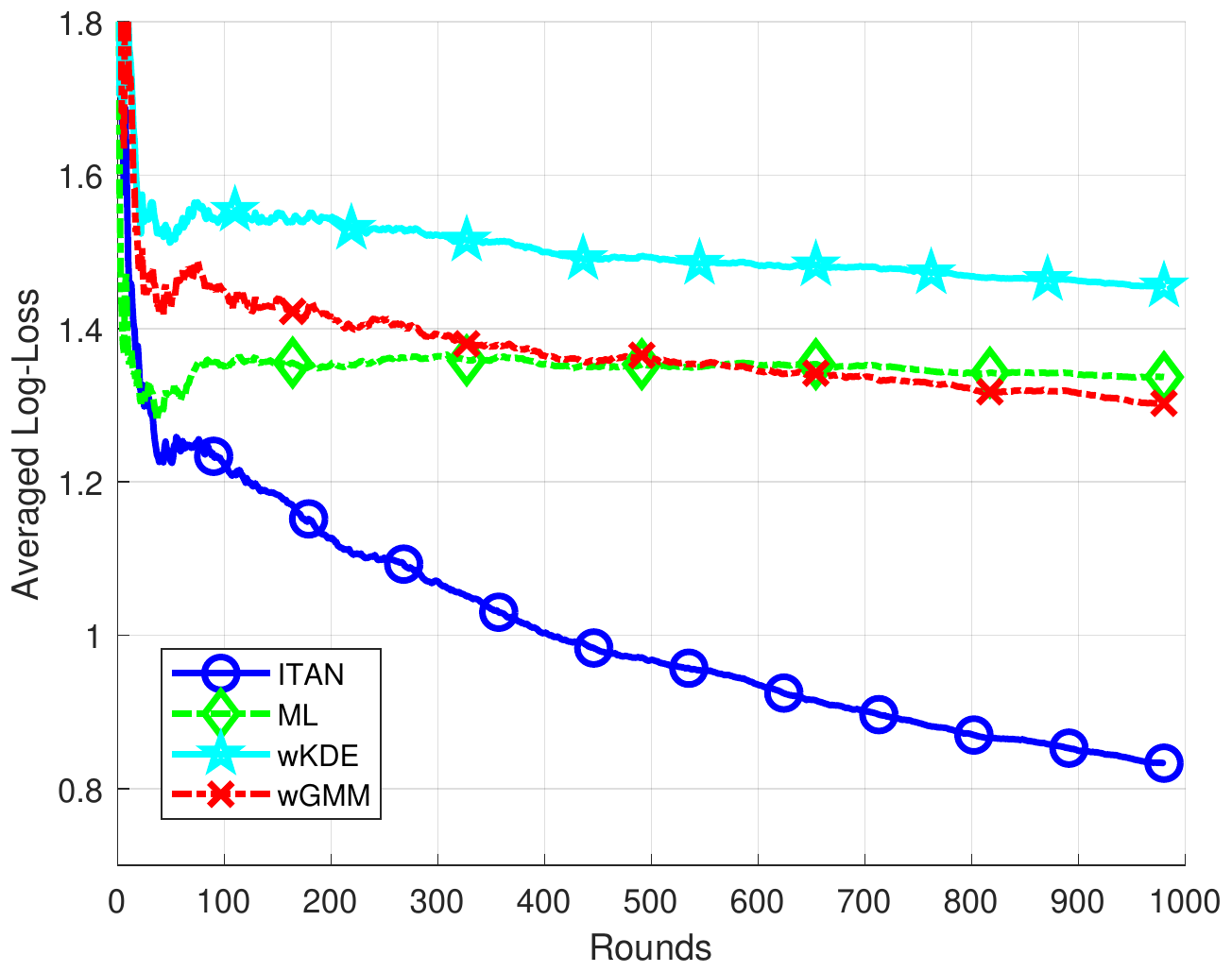}
			\caption{Density estimation loss of the algorithms over the datasets described in Section \ref{Ex:2}. The results are averaged over $10$ datasets.} \label{fig:ex2pdf}
		\end{subfigure}
		\hspace*{\fill} 
		\begin{subfigure}{0.3\textwidth}\vspace{-7mm}
			\includegraphics[width=\linewidth]{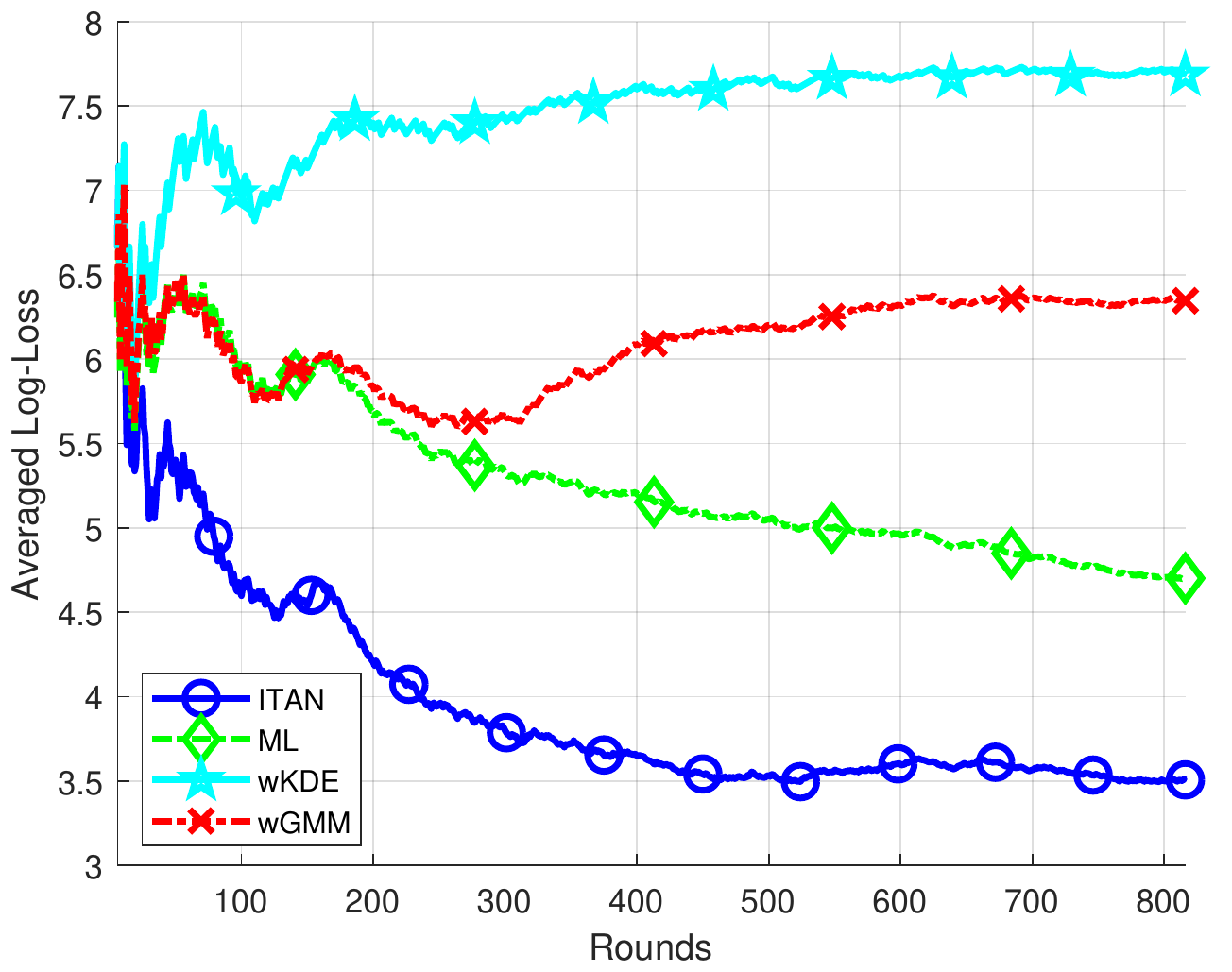}
			\caption{Density estimation loss of the algorithms over the Vehicle Silhouettes dataset.}
			\label{fig:ex3pdf}
		\end{subfigure}
		\begin{subfigure}{0.3\textwidth}
			\includegraphics[width=\linewidth]{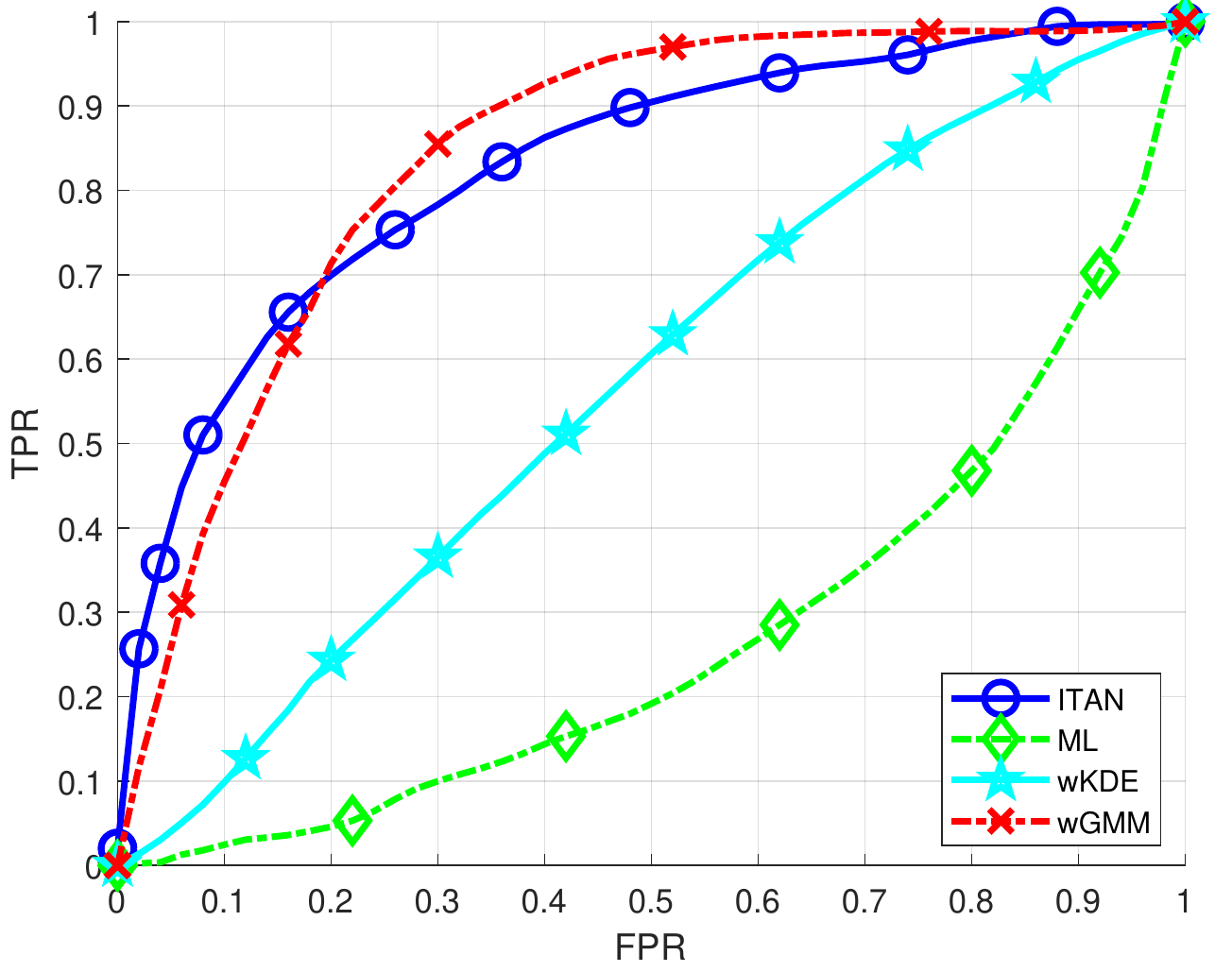}
			\caption{ROC curves of the algorithms over the datasets described in Section \ref{Ex:1}. The results are averaged over $10$ datasets.} \label{fig:ex1roc}
		\end{subfigure}
		\hspace*{\fill} 
		\begin{subfigure}{0.3\textwidth}
			\includegraphics[width=\linewidth]{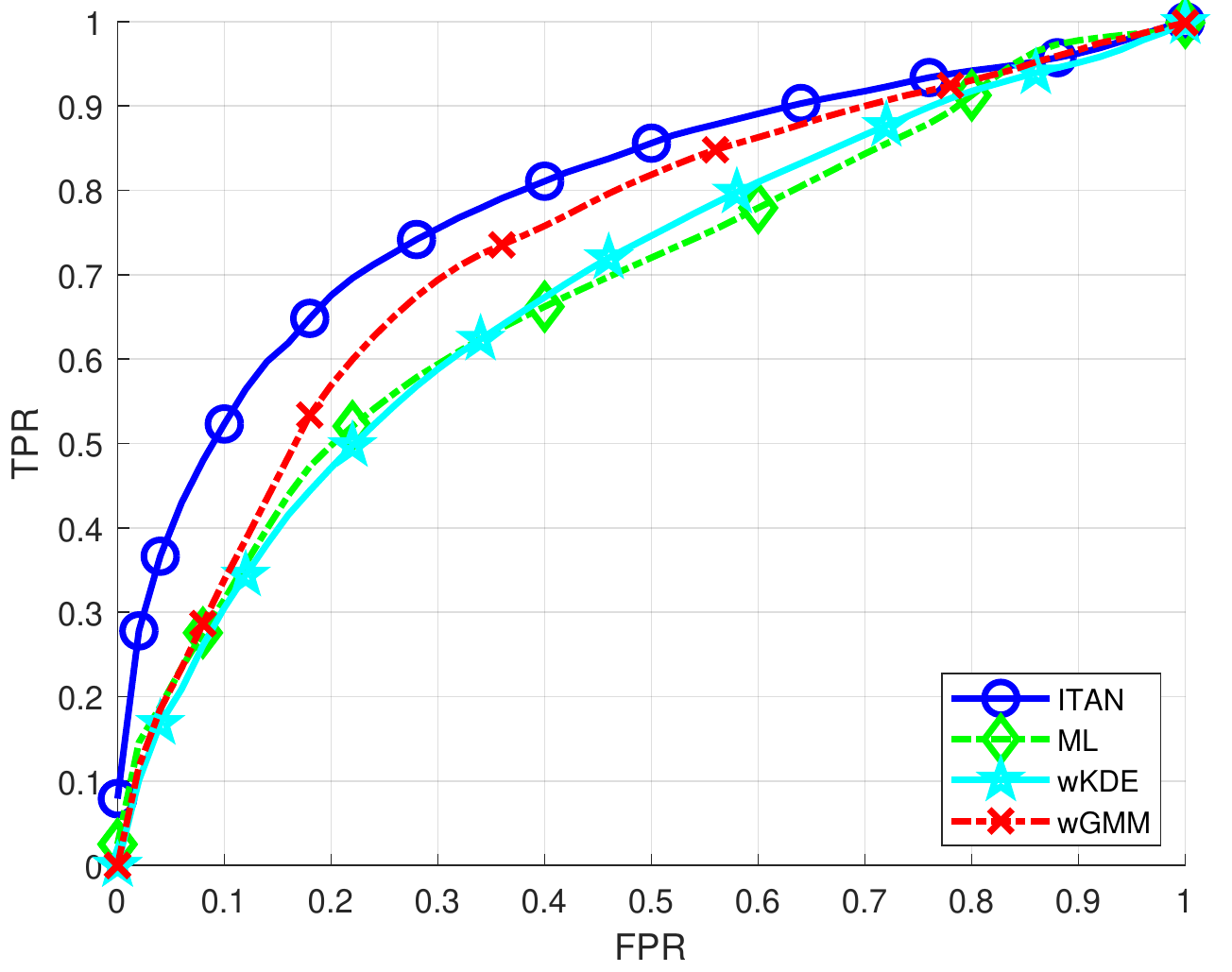}
			\caption{ROC curves of the algorithms over the datasets described in Section \ref{Ex:2}. The results are averaged over $10$ datasets.}
			\label{fig:ex2roc}
		\end{subfigure}
		\hspace*{\fill} 
		\begin{subfigure}{0.3\textwidth}\vspace{-4mm}
			\includegraphics[width=\linewidth]{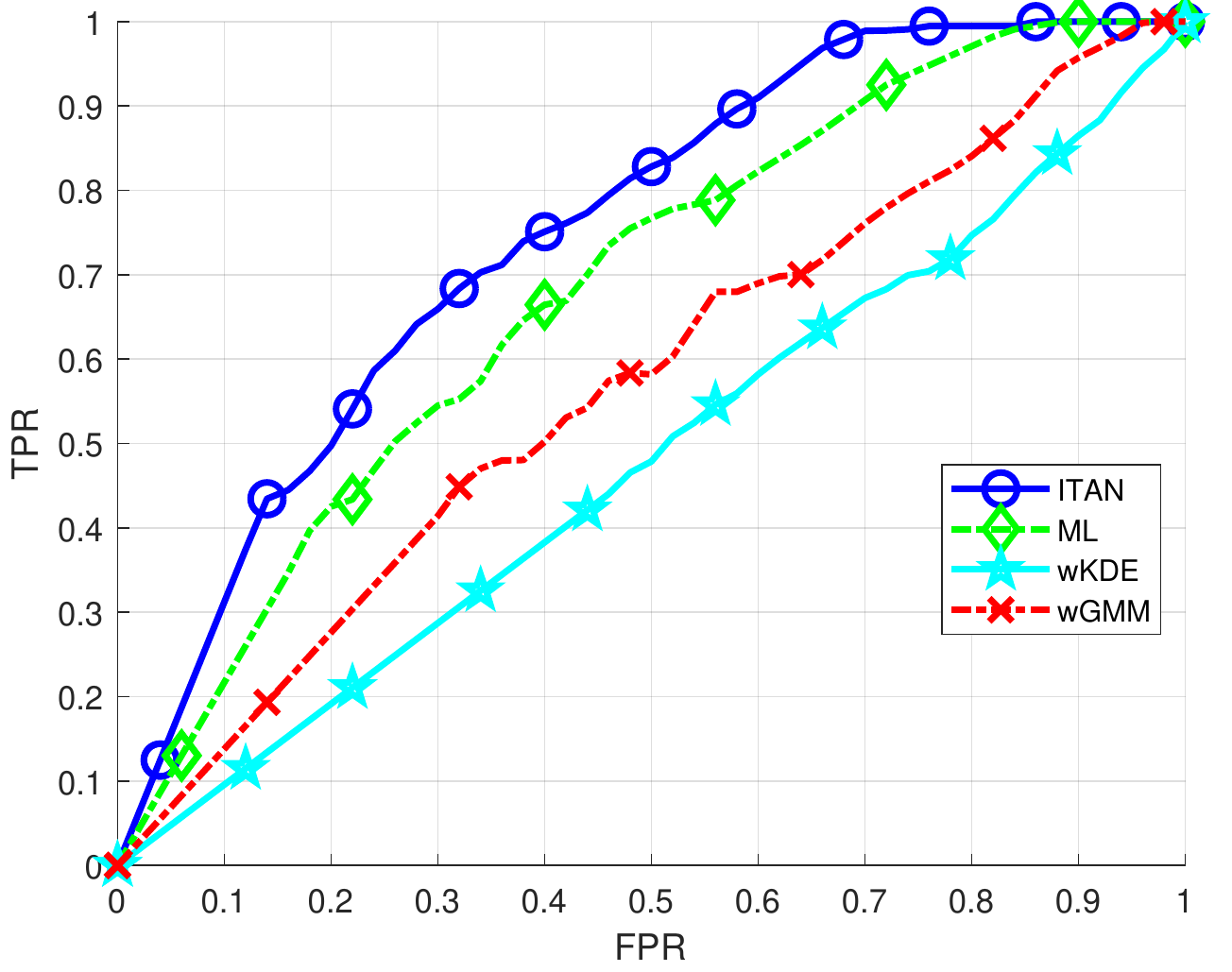}
			\caption{ROC curves of the algorithms over the Vehicle Silhouettes dataset.} 
			\label{fig:ex3roc}
		\end{subfigure}	
		\caption{The averaged density estimation loss and ROC curves of the algorithms over $3$ experiments.} 
	\end{figure*}

	As shown in Figs. \ref{fig:ex1pdf}, our algorithm achieves a significantly superior performance for the density estimation task. This superior performance was expected because in the dataset used for this experiment the components are far from each other. Hence, our tree can generate proper subspaces, which contain only the samples from one of the components of the underlying PDF, as shown in Fig. \ref{fig:combination}. For the anomaly detection task, as shown in Fig. \ref{fig:ex1roc}, our algorithm and \textit{wGMM} provide close performance, where \textit{ITAN} performs better in low FPRs and \textit{wGMM} provides superior performance in high FPRs. However, as shown in Table \ref{tab:ex1}, we achieve this performance with a significantly lower computational complexity. Comparing Fig. \ref{fig:ex1pdf} and Fig. \ref{fig:ex1roc} shows that while satisfactory log-loss performance is required for successful anomaly detection, it is not sufficient in general. For instance, while \textit{ML} algorithm performs as well as \textit{wKDE} and \textit{wGMM} in the log-loss sense, its anomaly detection performance is much worse than the others. In fact, labeling the samples exactly opposite of the suggestions of the \textit{ML} algorithm provides way better anomaly detection performance. This is because of the weakness of the model assumed by the \textit{ML} algorithm. This weak performance of the \textit{ML} algorithm was expected due to the underlying PDF of the normal and anomalous data. It can be also seen from Fig. \ref{fig:dataset1}. If we fit a single component Gaussian PDF to the normal samples shown in blue, roughly speaking, the anomalous samples shown in red will get the highest normality score when evaluated using our PDF.\par 
	
	In the next experiment, we compare the algorithms in a scenario, where the data cannot be modeled as a convex combination of Gaussian density functions.

	\subsection{Synthetic Arbitrary Distribution}\label{Ex:2}

	In this experiment, we have created $10$ datasets of length $1000$. In order to generate each sample, first its label is randomly determined to be ``normal" or ``anomalous" with probabilities of $0.9$ and $0.1$, respectively. The normal samples are $2$-dimensional vectors $\vec{x}_t=[x_{t,1},x_{t,2}]^T$ generated using the following distribution:
	\begin{equation}
	\begin{cases}
	f_{\text{normal}}({x}_{t,1})=\mathcal{U}(-1,1),\\

	f_{\text{normal}}({x}_{t,2})=\mathcal{U}(\sin{(\pi {x}_{t,1})},\sin{(\pi {x}_{t,1})}+0.2),
	\end{cases}
	\end{equation}
	where $\mathcal{U}(a,b)$ is the uniform distribution between $a$ and $b$. The anomalous samples are generated using the following distribution:
	\begin{equation}
	\begin{cases}
	f_{\text{normal}}({x}_{t,1})=\mathcal{U}(-1,1),\\

	f_{\text{normal}}({x}_{t,2})=\mathcal{U}(\cos{(\pi {x}_{t,1})},\cos{(\pi {x}_{t,1})}+0.2).
	\end{cases}
	\end{equation}
	Fig. \ref{fig:dataset2} shows the samples in one of the datasets used in this experiment to provide a clear visualization.\par
	
	Fig. \ref{fig:ex2pdf} shows the averaged accumulated loss of the algorithms averaged over 10 data sets. As shown in the figure, our algorithm outperforms the competitors for the density estimation task. This superior performance is due to the growing in time modeling power of out algorithm. The ROC curves of the algorithms for the anomaly detection task are shown in Fig. \ref{fig:ex2roc}. This figure shows that our algorithm provides superior anomaly detection performance as well. This superior performance is due to the better approximation of the underlying PDF made by our algorithm. The averaged log-loss performance, AUC results and running time of the algorithms in this experiment are summarized in Table. \ref{tab:ex2}.

	\begin{table}[t]
		\centering
		\caption{``Log-Loss", ``AUC" and ``Running Time" of the algorithms over the datasets described in Section \ref{Ex:2}. The AUC and Running Time values are in the format of ``mean value $\pm$ standard deviation. }
		\label{tab:ex2}
		\begin{tabular}{ |p{1.2cm}||p{1.1cm}|p{2.2cm}|p{2.4cm}|  }
			\hline
			Algorithm& Log-Loss & Area Under Curve  & Running Time (ms)\\
			\hline
			ITAN  & $0.833$& $0.7962 \pm 0.0757$    & $315.85 \pm 13.67$ \\
			wGMM   & $1.303$&$0.7381 \pm 0.0285$ & $7167.78 \pm 140.99$\\
			wKDE & $1.455$&   $0.6863 \pm 0.0321$  & $283.21 \pm 1.10$\\
			ML & $1.337$&   $0.6859 \pm 0.0323$  & $34.82 \pm 0.43$\\
			\hline
		\end{tabular}
	\end{table}
	
	\subsection{Real multi-class dataset}\label{Ex:3}
	
	In this experiment, we use Vehicle Silhouettes \cite{vehicle} dataset. This dataset contains $846$ samples. Each sample includes a $18$-dimensional feature vector extracted from an image of a vehicle. The labels are the vehicle class among $4$ possible classes of ``Opel", ``Saab", ``Bus" and ``Van". Our objective in this experiment is to detect the vehicles with ``Van" labels as our anomalies. Fig. \ref{fig:ex3pdf}, shows the density estimation loss of the opponents, based on the rounds in which they have observed ``normal" samples. Fig. \ref{fig:ex3roc} shows the ROC curves of the algorithms. As shown in the figures, our algorithm achieves a significantly superior performance in both density estimation and anomaly detection tasks over this dataset. The AUC results and running time of the algorithms are summarized in Table. 2.\par 
	
	Fig. \ref{fig:ex3pdf} shows that the performance of \textit{wGMM} highly depends on the stationarity of normal samples stream. The intrinsic abrupt change of the underlying model at around round $250$ has caused a heavy degradation in its density estimation performance. However, our algorithm shows a robust log-loss performance even in the case of non-stationarity. Fig. \ref{fig:ex3roc} shows that our algorithm achieves the best anomaly detection performance among the competitors. Note that \textit{ML} algorithm outperforms both \textit{wGMM} and \textit{wKDE} algorithms in both density estimation and anomaly detection tasks. This is because \textit{wGMM} and \textit{wKDE} suffer from overtraining due to the high dimensionality of the sample vectors and short time horizon of the experiment. However, due to the growing tree structure used in our algorithm, we significantly outperform the \textit{ML} algorithm and provide highly superior and more robust performance compared to the all other three algorithms.

	\begin{table}
		\label{tab:ex3}
		\centering
		\caption{``Log-Loss", ``AUC" and ``Running Time" of the algorithms over the Vehicle Silhouettes dataset. The running time values are in the format of ``mean value $\pm$ standard deviation". }
		\begin{tabular}{ |p{1.2cm}||p{1.1cm}|p{2.2cm}|p{2.4cm}|  }
			\hline
			Algorithm& Log-Loss & Area Under Curve  & Running Time (ms)\\
			\hline
			ITAN  & $3.507$& $0.7483 $    & $466.33 \pm 13.45$ \\
			wGMM   & $6.346$&$0.5682 $ & $4646.87 \pm 140.14$\\
			wKDE & $7.684$&   $0.4797 $  & $276.58 \pm 6.04$\\
			ML & $4.702$&   $0.6806 $  & $34.71 \pm 4.41$\\
			\hline
		\end{tabular}
	\end{table}
	
	\section{Concluding Remarks}\label{conclusion}
	We studied the sequential outlier detection problem and introduced a highly efficient algorithm to detect outliers or anomalous samples in a series of observations. We use a two-stage method, where we learn a PDF that best describes the normal samples, and decide on the label of the new observations based on their conformity to our model of normal samples. Our algorithm uses an incremental decision tree to split the observation space into subspaces whose number grow in time. A single component PDF is trained using the samples inside each subspace. These PDFs are adaptively combined to form our multi-modal density function. Using the aforementioned incremental decision tree, while avoiding overtraining issues, our modeling power increases as we observe more samples. We threshold our density function to decide on the label of new observations using an adaptive thresholding scheme. We prove performance upper bounds for both density estimation and thresholding stages of our algorithm. Due to our competitive algorithm framework, we refrain from any statistical assumptions on the underlying normal data and our performance bounds are guaranteed to hold in an individual sequence manner. Through extensive set of experiments involving synthetic and real datasets, we demonstrate the significant performance gains of our algorithm compared to the state-of-the-art methods.
	
	\bibliographystyle{IEEEbib}
	\bibliography{Ref}

\end{document}